\documentclass[journal]{IEEEtran}

\usepackage{bibentry}  

\usepackage{xcolor,soul,framed} 

\colorlet{shadecolor}{yellow}
\usepackage[pdftex]{graphicx}
\graphicspath{{../pdf/}{../jpeg/}}
\DeclareGraphicsExtensions{.pdf,.jpeg,.png}

\usepackage[cmex10]{amsmath}
\usepackage{arydshln}
\usepackage{mdwmath}
\usepackage{mdwtab}
\usepackage{eqparbox}
\usepackage{url}

\usepackage{booktabs}
\usepackage{multirow}
\usepackage{float}
\usepackage[thref,hyperref]{ntheorem}
\usepackage{hyperref}
\usepackage{cite}
\usepackage{amssymb}
\usepackage{ntheorem}

\usepackage{caption}
\hypersetup{linktocpage=false,colorlinks=true, linkcolor=black,anchorcolor=black,citecolor=black,draft=false}

\theoremseparator{:}
\theorembodyfont{}
\newtheorem{property}{Proposition}
\newenvironment{proof}{{ \it Proof:\quad}}{\hfill $\blacksquare$\par}

\begin{document}
\bstctlcite{IEEEexample:BSTcontrol}
    \title{Single Image Dehazing Using Scene Depth Ordering}
  \author{Pengyang Ling, 
      Huaian Chen,
      Xiao Tan,
      Yimeng Shan,
      and Yi Jin

  \thanks{P. Ling, H. Chen, and X. Tan are with the School of Engineering Science, University of Science and Technology of China, Anhui 230022, China 
  	
  	Y. Shan is with the College of Mechanical and Electrical Engineering, Nanjing University of Aeronautics and Astronautics, Nanjing 210016, China
  	
  	Y. Jin is with the School of Engineering Science and with the School of Data Science, University of Science and Technology of China, Anhui 230022, China (e-mail: jinyi08@ustc.edu.cn)
        } }


\maketitle

\begin{abstract}
Images captured in hazy weather generally suffer from quality degradation, and many dehazing methods have been developed to solve this problem. However, single image dehazing problem is still challenging due to its ill-posed nature. In this paper, we propose a depth order guided single image dehazing method, which utilizes depth order in hazy images to guide the dehazing process to achieve a similar depth perception in corresponding dehazing results. The consistency of depth perception ensures that the regions that look farther or closer in hazy images also appear farther or closer in the corresponding dehazing results, and thus effectively avoid the undesired visual effects. To achieve this goal, a simple yet effective strategy is proposed to extract the depth order in hazy images, which offers a reference for depth perception in hazy weather. Additionally, a depth order embedded transformation model is devised, which performs transmission estimation under the guidance of depth order to realize an unchanged depth order in the dehazing results. The extracted depth order provides a powerful global constraint for the dehazing process, which contributes to the efficient utilization of global information, thereby bringing an overall improvement in restoration quality. Extensive experiments demonstrate that the proposed method can better recover potential structure and vivid color with higher computational efficiency than the state-of-the-art dehazing methods.
\end{abstract}


%
\IEEEpeerreviewmaketitle


\section{Introduction}

\IEEEPARstart{D}{ue} to the presence of scattering caused by turbid medium, images captured in hazy weather are generally accompanied by visibility degradation. This damage is mainly caused by two aspects: i) the attenuation of reflected light from the objects in the medium, and ii) the addition of scattered light from the haze particles. Both of these factors will reduce the contrast and shift the colors in hazy images, leading to a decrease in the performance of subsequent advanced vision applications (such as traffic monitoring, target tracking, intelligent surveillance, etc.)\cite{app_1,app_2}. Therefore, to ensure the stability of vision-based systems, effective dehazing algorithms are in high demand.  
 
To obtain haze-free images, an early solution was to integrate the information from multiple images. For example, the polarization-based methods\cite{extra_2,extra_1,iterative} utilize multiple images taken under different polarization degrees to compensate for the lack of information in single image. Additionally, the dehazing methods based in multiple images captured on different weather scenes are also exploited in \cite{different_weather_1,different_weather_2,different_weather_3}. These methods can recover promising results through effective information integration. However, it is not trivial to obtain multiple images under severe weather conditions, resulting in the limited practical applications.

To address this problem, many researchers have attempted to develop single image dehazing methods to convert a hazy image with low contrast into a haze-free image rich in colors and details. As early attempts, some researches \cite{enhancement_1,enhancement_2,enhancement_3,enhancement_4} directly adopted traditional enhancement methods to conduct contrast enhancement, which help enhance the visibility of hazy images. However, the results of these enhancement-based approaches lack the constraint of haze degradation models, and thus some obvious drawbacks, such as over-saturation and color distortion, are usually produced in their dehazing results.

\begin{figure}
	\begin{center}
		\includegraphics[width=3.5in]{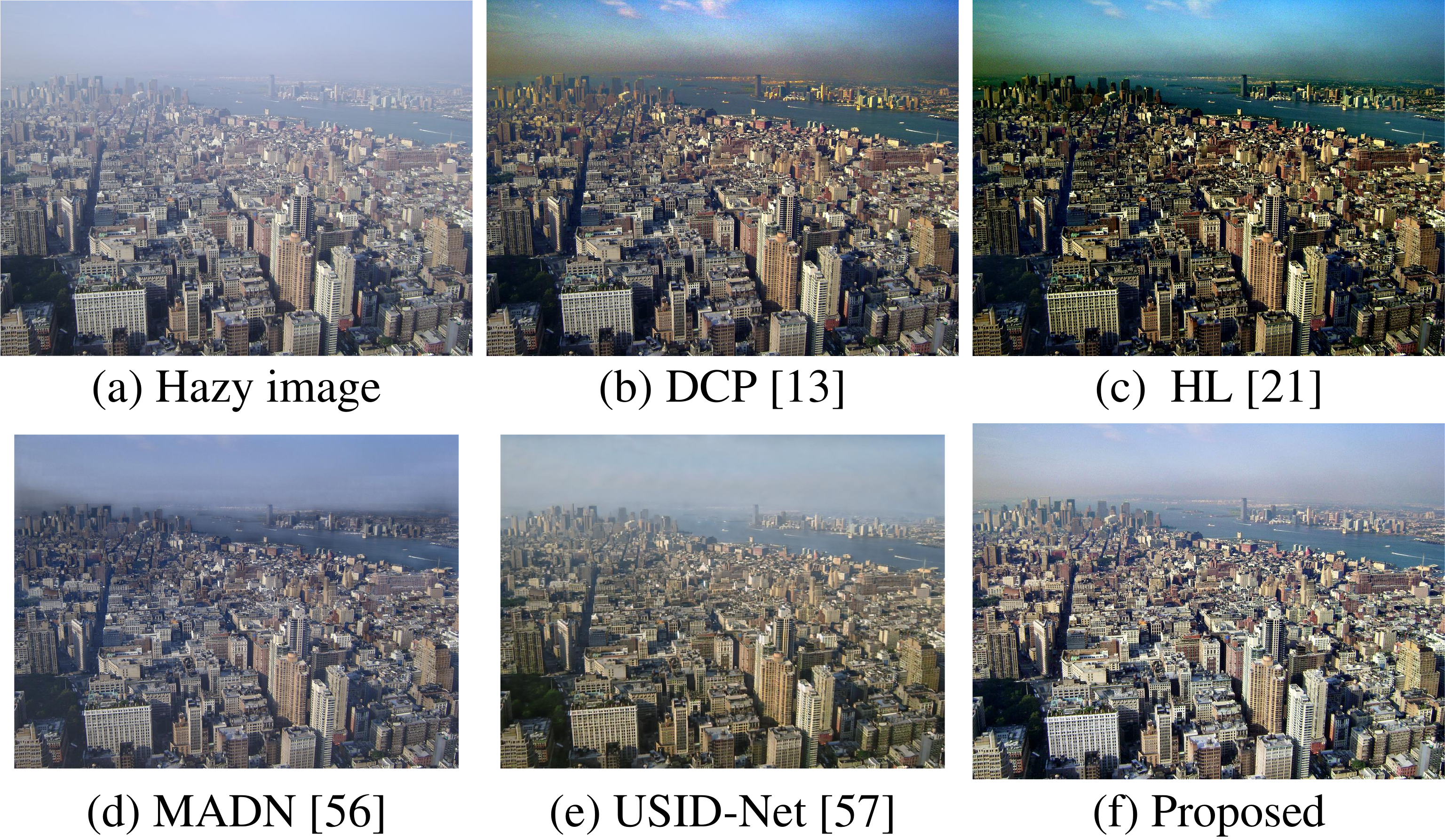}\\
		\caption{The dehazing results of different methods on challenging real-world hazy scene.}
		\label{fig_first}
	\end{center}
	\vspace{-1em} 
\end{figure}

As such, most later works are prone to develop dehazing methods based on reasonable assumptions or priors \cite{darkchannel,guided,cap,lt,fattal_1,fattal_2,ellipsoid_framework,cep,hl,idgcp}. These methods leverage priors as extra constraints to reduce the uncertainty of transmission estimation. Although these methods produce impressive results on many hazy scenes, they may generate unsatisfactory results for some scenes, in which the used priors or assumptions are invalid. Additionally, fusion strategies are exploited in \cite{fusion_1,fusion_2,defade} to merge the advantages of different methods. Typically, in these methods, a haze-free image is generated as the blend of multiple enhanced images via Laplacian pyramid decomposition. However, such a fusion scheme may encounter failures in dark or dense haze scenes.

Recently, convolutional neural network (CNN) based methods have made remarkable achievements \cite{dehazenet,mscnn,gated,gridhazenet,aodnet,rank,domainada,tcn,compact_dehazing,unsupervised_1,unsupervised_2,refinednet,towards} on single image dehazing task. By utilizing the excellent capability of CNN in learning latent features from haze-related data, a model with strong dehazing performance can be obtained. Regrettably, CNN-based methods require advanced hardware and vast training data, which limits their performance in real applications. In general, due to the ill-posed nature of single image dehazing, most existing dehazing methods are still affected by the degraded visibility to some extent, which can be seen in Figs. \ref{fig_first}(b)-(e).

In this paper, we propose a depth order guided single image dehazing method, which removes haze while retaining the original depth order in hazy images, and thus achieves a similar depth perception in the corresponding dehazing results. Specifically, a well-designed strategy is proposed to extract the depth order in hazy images, which offers a reference for depth perception in hazy weather. Then the extracted depth order is embedded into a transformation model for transmission estimation, which ensures an unchanged depth order while increasing the contrast. Finally, the dehazing results are obtained based on the estimated transmission, which show a similar depth perception as that in the corresponding hazy image. The consistency of depth perception provides a powerful global constraint for single image dehazing problem, which helps realize more effective haze removal based on the information of the whole image, thereby achieving a comprehensive visibility improvement. As an example, the dehazing result of Fig. \ref{fig_first}(a) obtained by the proposed method is shown in Fig. \ref{fig_first}(f). Moreover, it is worth mentioning that the computational complexity of the proposed method is also satisfactory due to the absence of pre-processing and iteration steps.  In summary, the main contributions of this work can be summarized as follows.

\begin{itemize}
	\item[$\bullet$]We propose a simple yet effective strategy to extract depth order in hazy images, which offers a reference for depth perception in hazy weather without estimating the specific depth values.
	\item[$\bullet$]We establish a transformation model for transmission estimation, which helps achieve a similar depth perception in dehazing results as that in corresponding hazy images and thus effectively relieves unreasonable visual effects.
	\item[$\bullet$]The proposed consistency of depth perception provides a powerful global constraint for the dehazing process and brings an efficient integration of global information. Experiments results validate the superiority of the proposed method in terms of both dehazing performance and computational complexity.
\end{itemize}

The remainder of this paper is organized as follows. Section II briefly reviews the related work. Section III presents the motivation and strategy for depth order extraction. Section IV gives the implementation details of the depth order guided dehazing model. Section V presents the experiments in various hazy scenes, and the conclusion is given in Section VI.
\section{Related Work} 
Generally, single image dehazing methods can be roughly classified into two categories: (i) prior-based methods, and (ii) data-driven methods.
\subsection {Prior-based Methods}
Although there is no determined solution for single hazy image input, reasonable solutions are limited. Prior-based methods \cite{darkchannel,guided,cap,lt,fattal_1,fattal_2,ellipsoid_framework,cep,hl,idgcp} utilize summarized regularities as additional constraints and greatly reduce the solution space. 

He \textit{et al.} \cite{darkchannel} revealed that most non-sky patches of clear images have at least one pixel with extremely low intensity in one color channel, known as dark channel prior (DCP). This prior greatly reduces the computational complexity of transmission estimation and produces an impressive result; however, the time cost of the refinement process is too expensive. To solve this problem, the guided filter was subsequently devised \cite{guided} to replace the soft matting in DCP and thus improved the overall algorithm efficiency. In \cite{cap}, Zhu \textit{et al.} pointed out that the values of brightness and saturation are effective cues for modeling depth in hazy images, and thus a simple linear model was trained to estimate the scene depth. However, the limitation of  independent variable values results in a limited dehazing performance in distant areas. Wang \textit{et al.} \cite{lt} utilized the linear transformation (LT) to estimate the minimum color component in haze-free images, which was assumed to be positively correlated with the imaging distance.

In \cite{fattal_1}, image dehazing was implemented under the assumption that surface shading and scene transmission are locally uncorrelated. This assumption is reliable when there is rich variation in surface shading and scene transmission but is not valid in dense haze areas with little variation. Later, Fattle \cite{fattal_2} proposed a color-line model to describe the 1-dimensional distribution of local pixels in RGB color space, which utilizes the distribution characteristics of local pixels to assist transmission estimation. Similarly, methods based on color ellipsoid were explored to extract the aggregation feature of local pixels \cite{ellipsoid_framework,cep}. Since the extracted features do not depend on a single pixel, these methods are generally more robust but are still limited to local information. To eliminate this limitation, methods based on global information were developed. Berman \textit{et al.} \cite{hl} observed that colors in a given haze-free image could be well approximated by a few hundred distinct colors that form tight clusters in RGB space. Thereby, the estimation for atmospheric light and transmission value can be conducted by using the non-local strategy. Ju \textit{et al.} \cite{idgcp} proposed gamma correction prior to embed depth information into gamma correction, and haze is then removed based on the virtual transformation from the whole hazy image.

\subsection {Data-driven Methods}
Recently, the blossom of deep learning has spawned an increasing number of CNN-based dehazing methods\cite{dehazenet,mscnn,gated,gridhazenet,aodnet,rank,domainada,tcn,compact_dehazing,unsupervised_1,unsupervised_2,refinednet,towards}. These methods train networks to learn a mapping function from hazy images to haze-free images to achieve haze removal. 

In \cite{dehazenet}, Cai \textit{et al.} proposed a dehaze network (DehazeNet), which first extracts priors-guided features and then merges them into an end-to-end network to estimate the final transmission map. Later, Ren \textit{et al.} \cite{mscnn} proposed a multi-scale CNN (MSCNN) to gradually optimize the transmission map from the coarse scale to the fine scale. In \cite{gated}, a gated fusion network (GFN)  was devised to fuse the features of three derived images from hazy images. Similarly, an attention-based multi-scale network (GridDehazeNet) \cite{gridhazenet} was proposed to enrich the diversity of image inputs for the improvement of  feature extraction in the following grid network. In \cite{aodnet}, Li \textit{et al.} proposed a lightweight CNN called AOD-Net, in which the transmission value and atmospheric light are integrated into one parameter to be estimated, which helps minimize the reconstruction error in final dehazing results. In\cite{rank}, Song \textit{et al.} proposed a ranking CNN to capture the structural and statistical features of hazy images simultaneously, which contributes to a more effective haze removal. To maintain excellent dehazing performance on real-world scenes, Shao \textit{et al.} \cite{domainada} designed a domain adaptation paradigm to reduce the impact caused by domain shift. To relieve the intensity-degradation of dark and shadow regions in dehazing process, Shin \textit{et al.} \cite{tcn} developed a triple-convolutional network (TCN), which performs targeted haze removal and image enhancement according to the degradation characteristics of different regions. For a balance of parameter numbers and dehazing performance, Wu \textit{et al.} \cite{compact_dehazing} proposed a compact dehazing network, in which contrastive learning was adopted to simultaneously extract features from both positive and negative samples. Additionally, some unsupervised \cite{unsupervised_1,unsupervised_2} and weakly supervised \cite{refinednet,towards} methods were also developed to alleviate the limitation caused by the lack of paired hazy and haze-free images in real scenes.

Recently, the transformer model, which shows strong capability in achieving long-range interactions through attention mechanism, has gradually become mainstream in high-level vision tasks. Accordingly, some researches were conducted to explore the potential of transformer-based architecture in image restoration. In \cite{transformer_dehaze_1}, Guo \textit{et al.} introduced modulation matrices to alleviate the feature inconsistency between CNNs and transformers, enabling great integration of local representation and global context. In \cite{transformer_dehaze_2}, Song \textit{et al.} improved the swim transformer \cite{swim_transformer} architecture according to the characteristic of dehazing task, termed DehazeFormer, which achieves advanced dehazing performance in various benchmarks under the supervised learning strategy. To reduce the complexity for higher efficiency, Qiu \textit {et al.} \cite{transformer_dehaze_3} proposed to apply Taylor expansion to approximate the softmax-attention, in which a multi-scale attention refinement mechanism is introduced for error correction. In \cite{transformer_dehaze_4}, Gui \textit{et al.} employed the transformer consisting of double decoders as the backbone, facilitating high-quality haze removal. Generally, these methods also follow the learning-based approach, and thus their performance highly depends on the diversity and quality of training data.

\section{DEPTH ORDER IN DEHAZING PROCESS} 
According to Koschmieder's law \cite{law}, the degradation process of hazy images can be described by the atmospheric scattering model as: 
\begin{equation}\label{eq.1}
{\emph{H}}^c(x,y) = {\emph{J}^c}(x,y) \cdot t(x,y) + {\emph{A}^c} \cdot (1 - t(x,y)),
\end{equation}
where $(x,y)$ is the coordinate of a pixel within the image, $c \in \{ r,g,b\} $ is the color channel, ${\emph{H}}$ is the hazy image, ${\emph{J}}$ is the corresponding haze-free image, ${\emph{A}}$ is the atmospheric light, and $t$ is the medium transmission describing the portion of the light that is not scattered. Eq. (\ref{eq.1}) can be transformed into
\begin{equation}\label{eq.2}
	{\emph{H}^c}(x,y) - {\emph{A}^c}= ({\emph{J}^c}(x,y) - {\emph{A}^c}) \cdot t(x,y).
\end{equation}
Eq. (\ref{eq.2}) indicates that haze will decrease the difference between the pixels and the atmospheric light in RGB color space. In this work, we take ${\theta}(x,y)$ to denote this color difference, which can be expressed as:
\begin{equation}\label{eq.3}
	\theta  = \sqrt {{{({I^r} - {A^r})}^2} + {{({I^g} - {A^g})}^2} + {{({I^b} - {A^b})}^2}} ,
\end{equation}
where $I \in \{ H,J\} $ denotes hazy images and haze-free images, respectively.
\begin{figure}
	\begin{center}
		\includegraphics[width=3.5in]{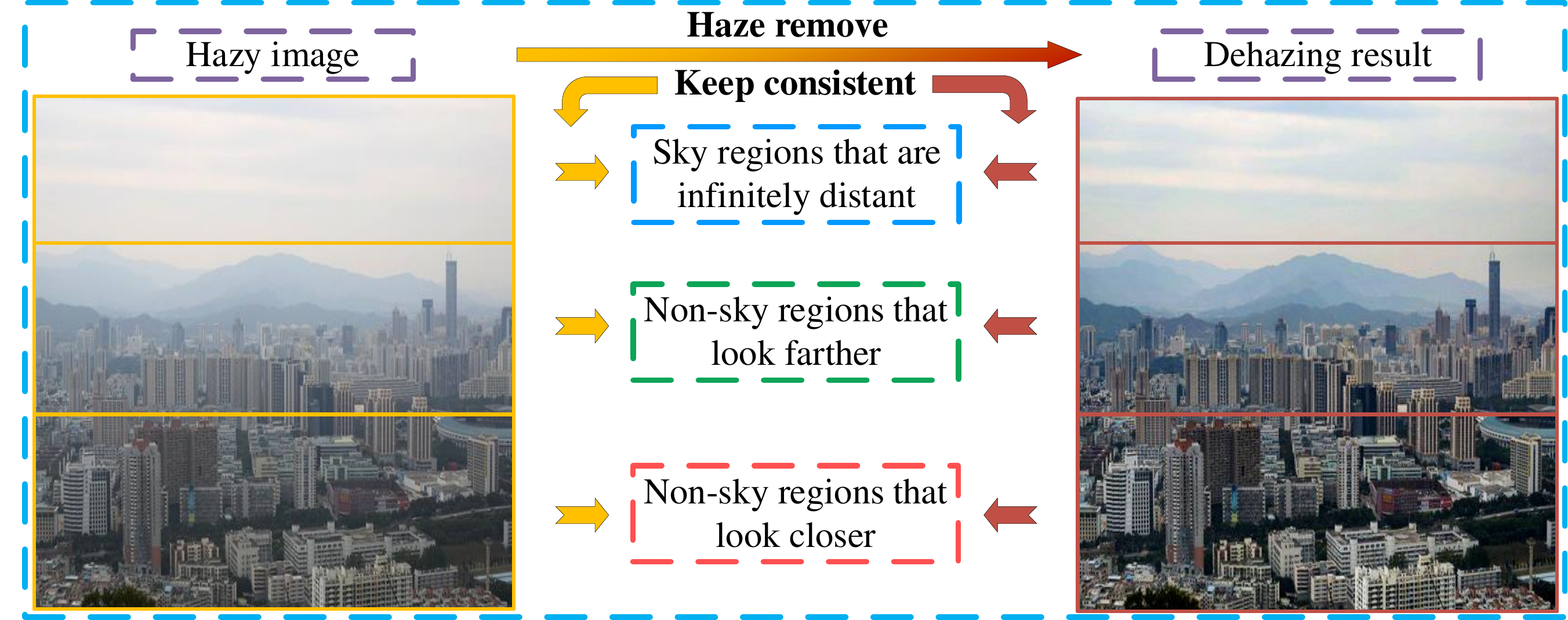}
		\caption{The consistency of depth perception between the hazy image and the dehazing result of the proposed method.}
		\label{fig_motivation}
	\end{center}
\end{figure}
Generally, assuming that medium is homogeneous, the transmission value can be modeled as:
\begin{equation}\label{eq.4}
	t(x,y) = {e^{ - \beta \cdot d(x,y)}},
\end{equation}
where $\beta$ is the scattering coefficient, and $d(x,y)$ is the distance from source to the camera at pixel position $(x,y)$. Since $\beta$ and $d(x,y)$ are both positive, the range of $t$ is (0,1). 

The purpose of single image dehazing is to restore ${\emph{J}^c}(x,y)$ from ${\emph{H}^c}(x,y)$, which can be expressed as:
\begin{equation}\label{eq.5}
	{\emph{J}^c}(x,y) = \frac{{{\emph{H}^c}(x,y) - {\emph{A}^c}}}{{t(x,y)}} + {\emph{A}^c}.
\end{equation}

\begin{figure*}
	\begin{center}
		\includegraphics[width=6.5in]{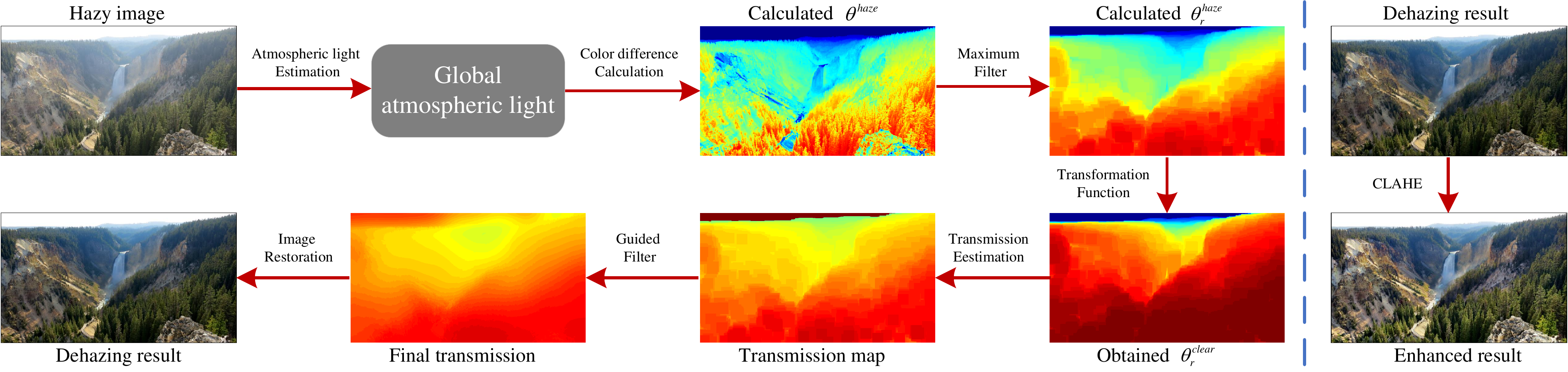}
		\caption{Overview of the proposed method.}
		\label{fig_overview}
	\end{center}
\end{figure*}

\subsection {Motivation} Although hazy images lose information about structure and details, they provide the real perception of humans in hazy weather. Meanwhile, notice that haze exists on both hazy and clear days, which provides the cue to perceive depth \cite{darkchannel,depth_cue}. Therefore, we expect that such depth perception is unchanged during the dehazing process, which ensures that the regions that look farther or closer in the hazy image also appear farther or closer in the corresponding dehazing result (see the example in Fig. \ref{fig_motivation}). This consistency of depth perception gives a powerful global constraint for the whole image, which helps achieve more effective haze removal. To avoid the direct estimation for depth value that may vary widely in outdoor scenes, in this work, we take depth order as the reference of depth perception, which only involves depth comparison.

Substituting Eq. (\ref{eq.3}) into Eq. (\ref{eq.2}), we have 
\begin{equation}\label{eq.6}
{\theta ^{haze}}(x,y) = {\theta ^{clear}}(x,y) \cdot t(x,y),
\end{equation}
where $\theta ^{haze}(x,y)$ and $\theta ^{clear}(x,y)$ denote the color difference between pixels and atmospheric light in hazy images and haze-free images, respectively. 
According to Eq. (\ref{eq.4}), Eq. (\ref{eq.6}) can be rewritten as:
\begin{equation}\label{eq.7}
	{\theta ^{haze}}(x,y)  = {\theta ^{clear}}(x,y) \cdot {e^{ - \beta \cdot d(x,y)}}.
\end{equation}
Eq. (\ref{eq.7}) reveals that the value of  ${\theta ^{haze}}(x,y)$ in hazy images is determined by there factors, i.e., ${\theta ^{clear}}(x,y)$ value in corresponding haze-free images, scattering coefficient $\beta $, and  depth $d(x,y)$. Given a hazy image, $\beta $ is a constant. Therefore, ${\theta ^{clear}}(x,y)$ and $d(x,y)$ jointly determine the order of ${\theta ^{haze}}(x,y)$ in hazy images. Note that the influence of depth $d(x,y)$ works with the form of an exponential function, which shows an explosive decay. Naturally, we expect that the value of $d(x,y)$ could dominate the order of ${\theta ^{haze}}(x,y)$ value. If it is feasible, according to Eq. (\ref{eq.7}), the reverse order of ${\theta ^{haze}}(x,y)$ is a good approximation for the order of scene depth $d(x,y)$, i.e., the smaller value of ${\theta ^{haze}}(x,y)$ is, the higher value of $d(x,y)$ is. The proposed method aims to employ the order of depth order as global constraint for an effective haze removal, the overview is depicted in Fig. \ref{fig_overview}.

\begin{figure*}
	\begin{center}
		\includegraphics[width=7in]{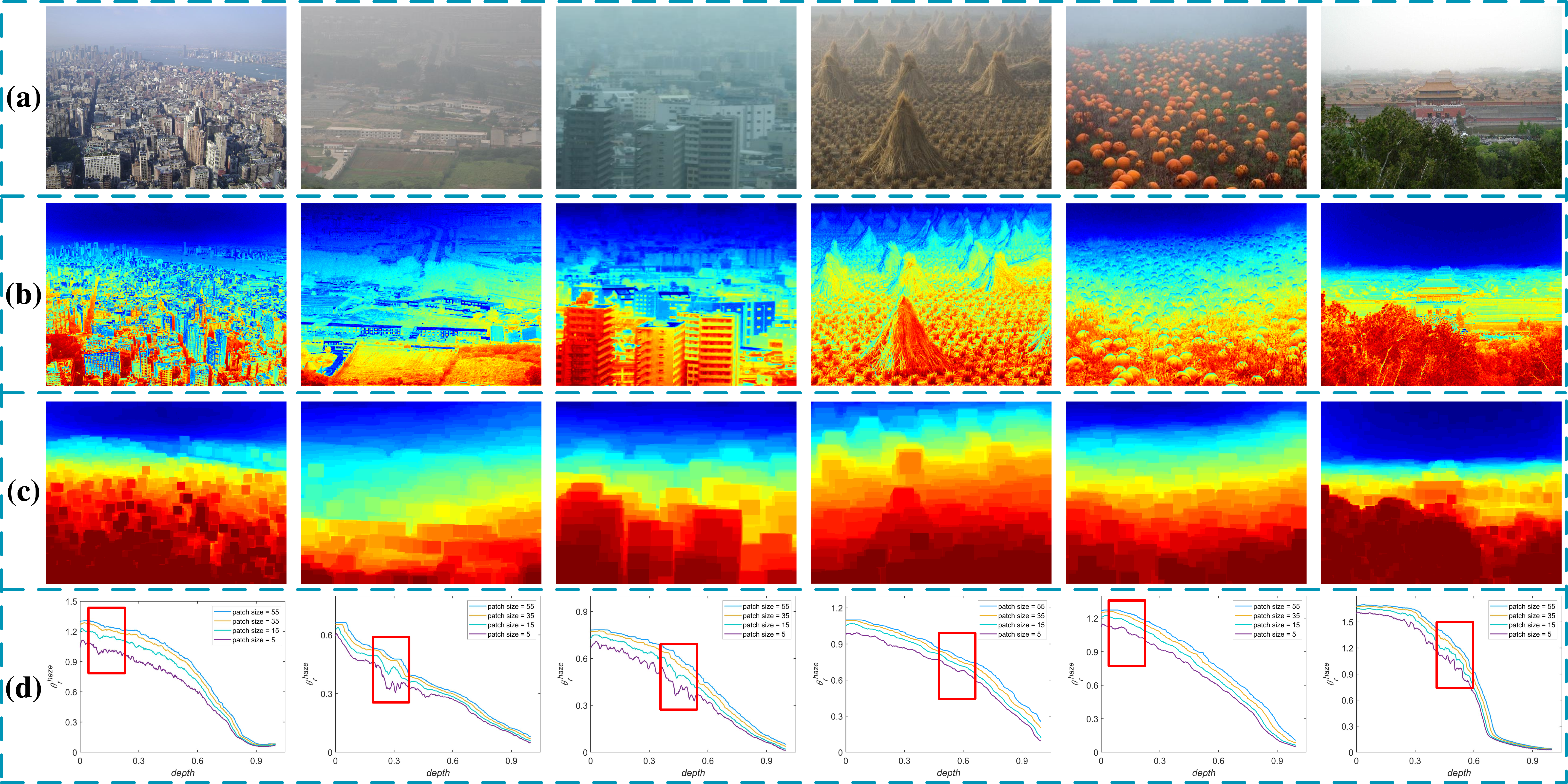}
		\caption{Examples of the relationship between $\theta ^{haze}(x,y)$ and scene depth in real-world hazy images.  \textbf{(a):} Hazy images.  \textbf{(b):} $\theta ^{haze}(x,y)$ obtained by Eq. (\ref{eq.3}).  \textbf{(c):} $\theta _r^{haze}(x,y)$ obtained by Eqs. (\ref{eq.3}) and (\ref{eq.12}) with $r=35$.  \textbf{(d):} The relationship between $\theta _r^{haze}(x,y)$ and scene depth under different patch sizes.}
		\label{fig_demonstration}
	\end{center}
\end{figure*}
\subsection {Depth Order Extraction}

Suppose that there are two pixels $({x_0},{y_0})$ and $({x_1},{y_1})$ in the haze-free image, and the corresponding color difference with atmospheric light are ${\theta ^{clear}}({x_0},{y_0})$ and ${\theta ^{clear}}({x_1},{y_1})$, respectively. Without loss of generality, the depth order is assumed to be
\begin{equation}\label{eq.8}
d({x_0},{y_0}) \ge d({x_1},{y_1}).
\end{equation}
The condition satisfying that the order of scene depth $d(x,y)$ can be approximated by the reverse order of ${\theta ^{haze}}(x,y)$ can be expressed as:
\begin{equation}\label{eq.9}
	{\theta ^{haze}}({x_0},{y_0}) \le {\theta ^{haze}}({x_1},{y_1}).
\end{equation}
According to Eq. (\ref{eq.7}) , we have
\begin{equation}\label{eq.10}
\frac{{{\theta ^{haze}}({x_0},{y_0})}}{{{\theta ^{haze}}({x_1},{y_1})}}\!=\! \frac{{{\theta ^{clear}}({x_0},{y_0})}}{{{\theta ^{clear}}({x_1},{y_1})}} \cdot {e^{ \! -\! \beta  \cdot (d({x_0},{y_0}) - d({x_1},{y_1}))}}.
\end{equation}
Note that $\beta $ is positive, thus inequality (\ref{eq.9}) is equivalent to
\begin{equation}\label{eq.11}
\beta  \cdot (d({x_0},{y_0}) - d({x_1},{y_1})) \ge ln(\frac{{{\theta ^{clear}}({x_0},{y_0})}}{{{\theta ^{clear}}({x_1},{y_1})}}).
\end{equation}
According to inequality (\ref{eq.8}), the left side of inequality (\ref{eq.11}) is non-negative. Thus inequality (\ref{eq.11}) obviously holds for every  ${\theta ^{clear}}({x_0},{y_0}) \le {\theta ^{clear}}({x_1},{y_1})$. In the case of  ${\theta ^{clear}}({x_0},{y_0})  >  {\theta ^{clear}}({x_1},{y_1})$, there are three factors that contribute to the satisfaction of inequality (\ref{eq.11}). First, a larger value of $\beta $. Second, a large value of the depth difference $d({x_0},{y_0}) - d({x_1},{y_1})$. Third, a smaller value of the ratio  $\frac{{{\theta ^{clear}}({x_0},{y_0})}}{{{\theta ^{clear}}({x_1},{y_1})}}$. The first and second factors correspond to the dense haze and sharp depth variation, respectively, which is easily satisfied in outdoor scenes with dense haze. Our goal is to minimize the value of the ratio  $\frac{{{\theta ^{clear}}({x_0},{y_0})}}{{{\theta ^{clear}}({x_1},{y_1})}}$, thereby facilitating the satisfaction of inequality (\ref{eq.11}) in this case.

Based on the assumption that local pixels share the same depth, we can obtain $d(x,y) = \mathop {max}\limits_{(m,n) \in \pi (x,y)}(d({m},{n})) $. Therefore, we replace $\theta$ with its local maximum value to focus on the pixels with maximum $\theta$ in local image patches, which can be expressed as:
\begin{equation}\label{eq.12}
	{\theta _r}(x,y) = \mathop {max}\limits_{(m,n) \in \pi (x,y)} (\theta (m,n)),
\end{equation}
where $\pi(x,y)$ is the local patch centered at pixel $(x,y)$ with patch size $r$.
From Eq. (\ref{eq.7}), those pixels with maximum $\theta^{haze}$ in hazy images also have maximum $\theta^{clear}$ in the corresponding clear images due to the locally constant depth.

According to Eq. (\ref{eq.3}), the maximum value of ${\theta _r^{clear}}({x_0},{y_0})$ is determined, i.e., $\sqrt 3$. By increasing the value of the denominator under a limited value of the numerator, the value of the ratio $\frac{{\theta _r^{clear}({x_0},{y_0})}}{{\theta _r^{clear}({x_1},{y_1})}}$ can be highly limited.

For a better understanding, considering the standard hazy weather with $\beta =1$ \cite{cap} and $ A =(1,1,1)$ \cite{standard}, for the extreme case that $\theta _r^{clear}({x_0},{y_0})=\sqrt 3$, inequality (\ref{eq.11}) can be expressed as:
\begin{equation}\label{eq.13}
	(d({x_0},{y_0}) - d({x_1},{y_1})) \ge ln(\frac{{\sqrt 3 }}{{{\theta _r^{clear}}({x_1},{y_1})}}).
\end{equation}
Note that haze-free images have rich color variations and surface shading in non-sky regions, which help increase $\theta _r^{clear}({x},{y})$ values. Statistical results \cite{darkchannel} demonstrate that there is at least one pixel, which has one channel close to zero, in the local patches for non-sky regions in haze-free images. According to Eq. (\ref{eq.3}),  $\theta _r^{clear}({x},{y})$ values of these pixels are larger than $1$, thus the sufficient condition for inequality (\ref{eq.13}) to be satisfied can be expressed as:
\begin{equation}\label{eq.14}
		(d({x_0},{y_0})\! -\! d({x_1},{y_1})) \ge ln(\frac{{\sqrt 3 }}{1}) \ge ln(\frac{{\sqrt 3 }}{{\theta _r^{clear}({x_1},{y_1})}}).
\end{equation}
Accordingly, in this case, the minimum depth difference identifying different non-sky regions is $ln(\sqrt 3 )=0.5493$. Meanwhile, due to the extremely small values of ${\theta _r^{haze}}({x},{y})$ in sky regions, ${\theta _r^{haze}}({x},{y})$ values also help distinguish the sky regions from the non-sky regions, which helps the targeted suppression of noise amplification in sky regions. 
Therefore, the reverse order of $\theta_r^{haze}$ is a good approximation of the order of depth, i.e., the smaller value of $\theta_{r}^{haze}(x,y)$ is, the higher value of depth $d(x,y)$ is.

\subsection {Effectiveness Demonstration}

For an intuitive demonstration, six different types of hazy images are collected and shown in Fig. \ref{fig_demonstration}(a). The corresponding $\theta ^{haze}(x,y)$ and $\theta _r^{haze}(x,y)$ values are presented in Figs. \ref{fig_demonstration}(b) and \ref{fig_demonstration}(c), respectively. The relationship between $\theta _r^{haze}(x,y)$  and scene depth under different patch sizes are given in Fig. \ref{fig_demonstration}(d). Note that, since the ground-truth scene depth value in Fig. \ref{fig_demonstration}(a) are not available, the abscissa in Fig. \ref{fig_demonstration}(d) denotes the row order (0 for the bottom row and 1 the for the top row) of each hazy image in Fig. \ref{fig_demonstration}(a), and the ordinate denotes the mean value of $\theta _r^{haze}(x,y)$ for each row. Consequently, the relationship between $\theta _r^{haze}(x,y)$ and scene depth can be indirectly described by the relationship between $\theta _r^{haze}(x,y)$ and row order that is approximately positively related with the scene depth (see Fig. \ref{fig_demonstration}(a)).

From Figs. \ref{fig_demonstration}(b)-(c), it is observed that the $\theta ^{haze}(x,y)$ increases as depth decreases in all hazy images, and Eq. (\ref{eq.12}) can markedly strengthen such relationship, which confirms the effectiveness of the proposed extracted strategy. A clearer illustration of the capacity of Eq. (\ref{eq.12}) is shown in Fig. \ref{fig_demonstration}(d), it can be seen that a large patch size can effectively erase the fluctuation in the variations (see the red rectangles in Fig. \ref{fig_demonstration}(d)), which helps bring a more effective extraction for depth order. This is because the ratio  $\frac{{\theta _r^{clear}({x_0},{y_0})}}{{\theta _r^{clear}({x_1},{y_1})}}$ can be better limited under a larger patch size, as analyzed before. The extracted depth order gives the reference of depth perception, which provides the bridge of the conversion from hazy images to haze-free images.

To explore the relationship between the extracted depth order in hazy image and corresponding haze-free image, Spearman's rank correlation coefficient $\rho $ \cite{spearman} is adopted, which can be expressed as:
\begin{equation}\label{eq.15}
\rho  = \frac{{\sum\nolimits_i {({X_i} - \bar X)({Y_i} - \bar Y)} }}{{\sqrt {{{\sum\nolimits_i {{{({X_i} - \bar X)}^2}\sum\nolimits_i {({Y_i} - \bar Y)} } }^2}} }}
\end{equation}
where $X$ and $Y$ are the rank of two sequences, and $ \bar X$ and $ \bar Y $ are the corresponding average values. For example, the rank of sequence ${\rm{\{  0}}{\rm{.2, 1}}{\rm{.2, 0}}{\rm{.4\} }}$ is ${\rm{\{ 1, 3, 2\}  }} $. Note that the greater the absolute value of $\rho $ is, the more correlated the two sequences are. Particularly, $|\rho| = 1$ denotes that they are monotonously related. We calculated the $\rho $ values of depth order, which are obtained by using Eqs. (\ref{eq.3}) and (\ref{eq.12}) with path size of 15 ( the same as that in \cite{darkchannel} ), in 500 synthetic hazy images and corresponding ground-truth images from the SOTS Dataset \cite{RESIDE}. The corresponding statistical histogram and cumulative density are given in Figs. \ref{fig_cdf_spearman}(a) and \ref{fig_cdf_spearman}(b), respectively.

It can be concluded from Fig. \ref{fig_cdf_spearman}(b) that $90\% $ of the samples have high values of $\rho>0.8$, which indicates the high correlation between the extracted depth order of hazy images and corresponding ground-truth images. As negative examples, Fig. \ref{fig_demonstration_spearman} gives the comparison between ground-truth images and the dehazing results with poor restoration quality. It is seen that a smaller $\rho $ value usually implies a lower restoration quality, which can be attributed to failure of ensuring depth perception consistency between hazy images and corresponding dehazing results. Therefore, it is necessary to retain the depth order of hazy images during dehazing process, which helps to achieve effective haze removal under the powerful global constraint of consistency of depth perception.
\begin{figure}
	\begin{center}
		\includegraphics[width=3.5in]{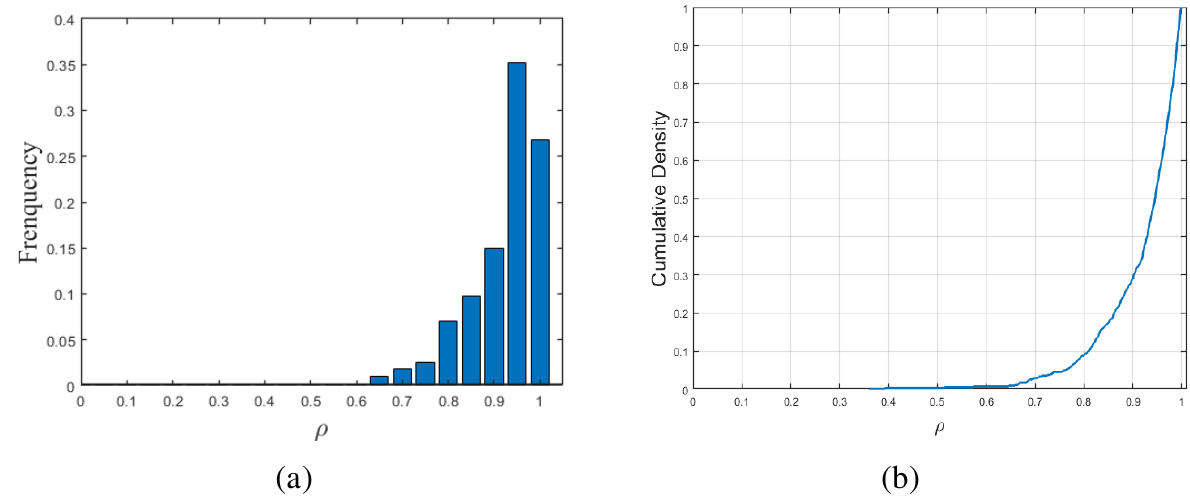}
		\caption{Frequency and cumulative density statistics of the $\rho$ values on 500 paired hazy images and ground-truth images. }
		\label{fig_cdf_spearman}
	\end{center}
\end{figure} 
\begin{figure}
	\begin{center}
		\includegraphics[width=3.5in]{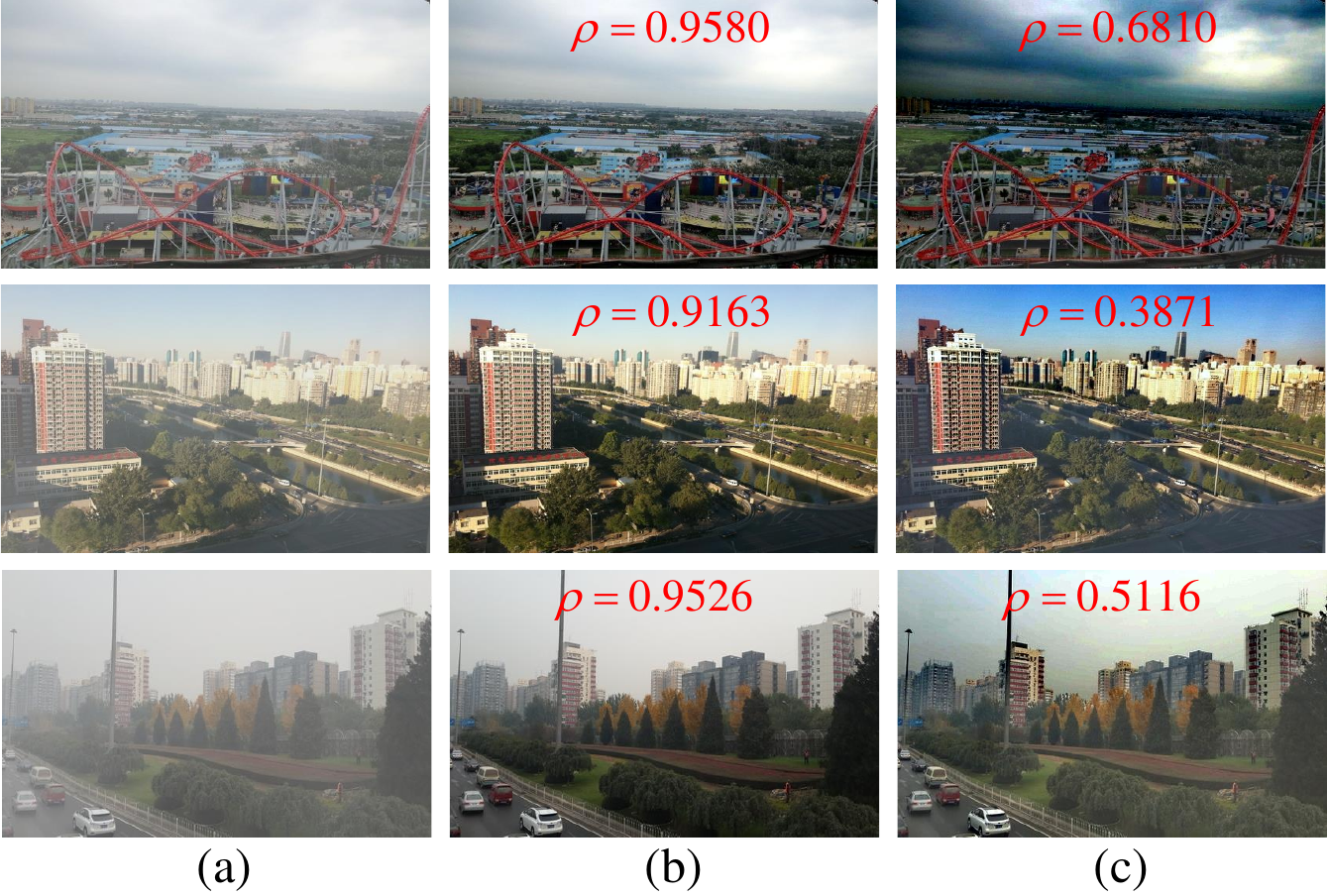}
		\caption{Comparison between the ground-truth images and the dehazing results with low $\rho$ values. \textbf{(a):} Hazy images. \textbf{(b):} Ground-truth images. \textbf{(c):} Dehazing results by HL \cite{hl}.}
		\label{fig_demonstration_spearman}
	\end{center}
\end{figure} 
\section{HAZE REMOVAL BASED ON DEPTH ORDER}

In this section, based on the atmospheric scattering model and the extracted depth perception, a depth order guided dehazing method is developed. The consistency of depth perception that offers the global constraint is embedded into the transmission estimation via a finely devised transformation model, and thus provides effective global guidance for image dehazing, resulting in higher dehazing performance.  

\subsection {Transformation Model for Transmission Estimation}
Given a hazy image, supposing that atmospheric light ${\emph{A}}$ is known,  ${\theta _r^{haze}}(x,y)$ values can be obtained from Eqs. (\ref{eq.3}) and (\ref{eq.12}). Therefore, according to Eq. (\ref{eq.6}), the estimation for transmission is equivalent to the estimation for ${\theta _r^{clear}}(x,y)$. Based on the analysis in last section, the design goals of the estimation for ${\theta _r^{clear}}(x,y)$ can be summarized as follows:

\begin{itemize}
	\item[$1)$]The ${\theta _r^{clear}}(x,y)$ values in dehazing results must be larger than the corresponding ${\theta _r^{haze}}(x,y)$ values in hazy images, as revealed by Eq. (\ref{eq.6}).
	\item[$2)$]The order of ${\theta}_r^{clear}(x,y)$ values should remain consistent with the order of ${\theta}_r^{haze}(x,y)$ values, thereby ensuring that the depth perception in the dehazing results is consistent with that in corresponding hazy images.
\end{itemize}

Design goal 1) can be expressed as
\begin{equation}\label{eq.16}
	\theta _r^{haze}(x,y) \le \theta _r^{clear}(x,y).
\end{equation}
Since $\theta _r^{haze}(x,y)$ is known, inequality (\ref{eq.16}) gives lower boundary of $\theta _r^{clear}(x,y)$. For a given image, let  $\hat \theta _r^{haze}$ and $\hat \theta _r^{clear}$ denote the maximum values of ${\theta _r^{haze}}(x,y)$ and ${\theta _r^{clear}}(x,y)$, respectively, which can be expressed as:
\begin{equation}\label{eq.17}
\hat \theta _r^{haze}\!=\! max(\theta _r^{haze}(x,y))\!\le \! max(\theta _r^{clear}(x,y))\!=\! \hat \theta _r^{clear}.
\end{equation}
From inequalities (\ref{eq.16}) and (\ref{eq.17}), the lower and upper boundaries of $\theta _r^{clear}(x,y)$ are obtained as:
\begin{equation}\label{eq.18}
\theta _r^{haze}(x,y) \le \theta _r^{clear}(x,y) \le \hat \theta _r^{clear}.
\end{equation}
Furthermore, inequality (\ref{eq.18}) can be converted into a linear interpolation of the two boundary values by
\begin{equation}\label{eq.19}
\theta _r^{clear}(x,y) = \theta _r^{haze}(x,y) \cdot (1 - w(x,y)) + \hat \theta _r^{clear} \cdot w(x,y),
\end{equation}
where  $w(x,y)$ is the interpolation coefficient with a range of $[0,1]$. Given values of $\theta _r^{haze}(x,y)$, we define $w(x,y)$ that varies with pixels as:
\begin{equation}\label{eq.20}
w(x,y) = \phi (\theta _r^{haze}(x,y)),
\end{equation}
where $\phi ( \cdot )$ is the weight function that returns the value of interpolation coefficient. For an unchanged order of ${\theta}(x,y)$, the result of Eq. (\ref{eq.19}) must meet the design goal 2). Indeed, we have the following property:

\begin{property}\label{property_1}
For every two pixels $({x_0},{y_0})$ and $({x_1},{y_1})$ with the order of $\theta _r^{haze}({x_0},{y_0}) \le \theta _r^{haze}({x_1},{y_1})$, the monotonically increasing weight function $\phi ( \cdot )$ is the sufficient condition for the results of Eq. (\ref{eq.19}) to satisfy $\theta _r^{clear}({x_0},{y_0}) \le \theta _r^{clear}({x_1},{y_1})$.
\end{property}

\begin{proof}	 
When the weight function $\phi ( \cdot )$  monotonically increases, from Eq. (\ref{eq.20}), the order of the interpolation coefficient can be obtained as:  
\begin{equation}\label{eq.21}
w({x_0},{y_0}) \le w({x_1},{y_1}).
\end{equation} 
From Eq. (\ref{eq.17}), we have
\begin{equation}\label{eq.22}
\theta _r^{haze}({x_0},{y_0}) \le \hat \theta _r^{clear}.
\end{equation}
According to Rearrangement Inequality, from inequalities (\ref{eq.21}) and (\ref{eq.22}), we obtain
\begin{equation}\label{eq.23}
	\begin{array}{l}
		\hat \theta _r^{clear} \cdot w({x_0},{y_0}) + \theta _r^{haze}({x_0},{y_0}) \cdot w({x_1},{y_1})\\
		\le \hat \theta _r^{clear} \cdot w({x_1},{y_1}) + \theta _r^{haze}({x_0},{y_0}) \cdot w({x_0},{y_0}).
	\end{array}
\end{equation}
Note that the range of $w({x_1},{y_1})$  is [0,1] and $\theta _r^{haze}({x_0},{y_0}) \le \theta _r^{haze}({x_1},{y_1})$, thus we have
\begin{equation}\label{eq.24}
\begin{array}{l}
	\theta _r^{haze}({x_0},{y_0}) \cdot (1 - w({x_1},{y_1}))\\
	\le \theta _r^{haze}({x_1},{y_1}) \cdot (1 - w({x_1},{y_1})).
\end{array}
\end{equation}
From  inequalities (\ref{eq.23}) and (\ref{eq.24}),  we obtain 
\begin{equation}\label{eq.25}
\begin{array}{l}
	\theta _r^{haze}({x_0},{y_0})\! \cdot\! (1 - w({x_0},{y_0}))\! +\! \hat \theta _r^{clear} \cdot w({x_0},{y_0})\\
	\le \theta _r^{haze}({x_1},{y_1})\! \cdot\! (1 - w({x_1},{y_1}))\! +\! \hat \theta _r^{clear} \cdot w({x_1},{y_1}).
\end{array}
\end{equation}
Substituting inequality (\ref{eq.25}) into Eq. (\ref{eq.19}), the expected order of $\theta _r^{clear}({x},{y})$ in the dehazing result can be obtained as:
\begin{equation}\label{eq.26}
\theta _r^{clear}({x_0},{y_0}) \le \theta _r^{clear}({x_1},{y_1}).
\end{equation}
\end{proof}

Since $\theta _r^{haze}(x,y)$ value varies from hazy images, normalization operation is adopted to unify $\theta _r^{haze}(x,y)$, which can be expressed as:
\begin{equation}\label{eq.27}
   z(x,y) = \frac{{\theta _r^{haze}(x,y) - min(\theta _r^{haze}(x,y))}}{{max(\theta _r^{haze}(x,y)) - min(\theta _r^{haze}(x,y))}}.
\end{equation}

According to proposition \ref{property_1}, all the monotonically increasing functions can be selected as weight function $\phi ( \cdot )$; however, the dehazing performance of these functions are different. Consequently, three functions with different shapes (convex function, linear function, and concave function) are discussed, which can be described as:
\begin{equation}\label{eq.28}
{\phi _1}(z) = z \cdot (2 - z),
\end{equation}
\begin{equation}\label{eq.29}
	{\phi _2}(z) = z, 
\end{equation}
\begin{equation}\label{eq.30}
	{\phi _3}(z) = z^2.
\end{equation}
Eqs. (\ref{eq.28})-(\ref{eq.30})  are visualized in Fig. \ref{fig_function_choice_1}, and the dehazing results with a rough $\hat \theta _r^{clear}$ (set as 1.2 times of $\hat \theta _r^{haze}$) using these functions are presented in Figs. \ref{fig_function_choice_2}(b)-(d). It is seen that a large value produced by function ${\phi _1}$ helps remove the dense haze, yet appears to amplify the noise in distant regions (see the red rectangles in Fig. \ref{fig_function_choice_2}(b)). Meanwhile, a small value produced by function ${\phi _3}$ contributes to suppress the noise amplification, but introduces haze residues in dense haze regions (see the blue rectangles in Fig. \ref{fig_function_choice_2}(d)). By comparison, function ${\phi _2}$ achieves a preferable balance between haze removal and noise amplification (see Fig. \ref{fig_function_choice_2}(c)). Therefore, in this work, function ${\phi _2(\cdot)}$ is adopted as the weight function $\phi ( \cdot )$.

\subsection {Global Optimization Using Boundary Constraints}

The global parameter $\hat \theta _r^{clear}$, which varies from images, has a significant impact on the restoration quality. Figs. \ref{fig_global_optimization}
(b)-(e) show the dehazing results generated by using different $\hat \theta _r^{clear}$ values, from which, we have the following observation. On the one hand, a larger $\hat \theta _r^{clear}$ value helps increase the contrast of the whole image. On the other hand, over-saturation occurs in some regions when $\hat \theta _r^{clear}$ value is too large. Accordingly, $\hat \theta _r^{clear}$ should be finely solved from the whole image. Note that the dehazing results must be within the range of [0, 1], which yields a boundary for the transmission. According to Eq. (\ref{eq.5}), the transmission value corresponding dehazing result $\emph{J}$ reaching the boundary is
\begin{equation}\label{eq.31}
	{t_b}(x,y)\!=\!\mathop {max}\limits_{c \in \{ r,g,b\} }\!(max(\frac{{{H^c}(x,y)- {A^c}}}{{0 - {A^c}}},\frac{{{H^c}(x,y)-{A^c}}}{{1- {A^c}}})\!).
\end{equation}
Substituting Eqs. (\ref{eq.31}), (\ref{eq.19}), and (\ref{eq.20}) into Eq. (\ref{eq.6}), the corresponding $\hat \theta _r^{clear}$ value for each pixel reaching the boundary can be obtained as:
\begin{equation}\label{eq.32}
\begin{split}
		 \theta _b^{clear}(x,y) = {\left( {{t_b}(x,y) \cdot {\phi _2}(z(x,y))} \right)^{ - 1}} \cdot \theta _r^{haze}(x,y)\\
		\qquad \qquad \qquad  \cdot \left( {1 - {t_b}(x,y) + {t_b}(x,y) \cdot {\phi _2}(z(x,y))} \right).
\end{split}
\end{equation}

By sorting the values obtained from Eq. (\ref{eq.32}), the $\hat \theta _r^{clear}$ values corresponding to different degrees of boundary overflow can be expressed as: 
\begin{equation}\label{eq.33}
{\theta _\varepsilon } = SortP( \theta _b^{clear},\varepsilon ),
\end{equation}
where $SortP(\cdot)$ returns the percentiles of the elements in $ \theta _b^{clear}$ for the percentages $\varepsilon$. The examples of this function are given in Fig. \ref{fig_global_optimization}(f). In this work, we set $\varepsilon$ as 0.02 to make a trade-off between high contrast and low information loss, which corresponds to the dehazing results with 2\% pixels reaching the boundary of [0,1]. In addition, to meet the design criterion 1),  $\hat \theta _r^{clear}$  must be larger than  $\hat \theta _r^{haze}$. Thus, the global parameter is finally set as:
\begin{equation}\label{eq.34}
\hat \theta _r^{clear} = max({\theta _\varepsilon },\hat \theta _r^{haze}).
\end{equation}
The dehazing results using  $\hat \theta _r^{clear}$ generated by Eq. (\ref{eq.34}) are shown in Fig. \ref{fig_global_optimization}(g). It can be observed that global optimization can reliably improve the dehazing performance while avoid over-saturation, as expected.

\begin{figure}
	\begin{center}
		\includegraphics[width=2in]{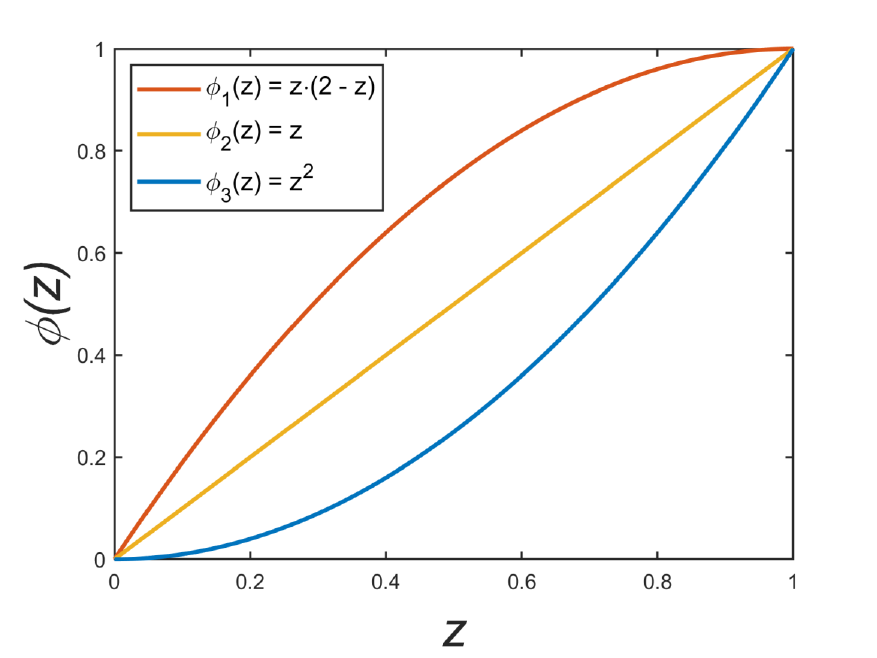}
		\caption{Visualization of Eqs. (\ref{eq.28})-(\ref{eq.30}).}\label{fig_function_choice_1}
	\end{center}
\end{figure}
\begin{figure}
	\begin{center}
		\includegraphics[width=3.5in]{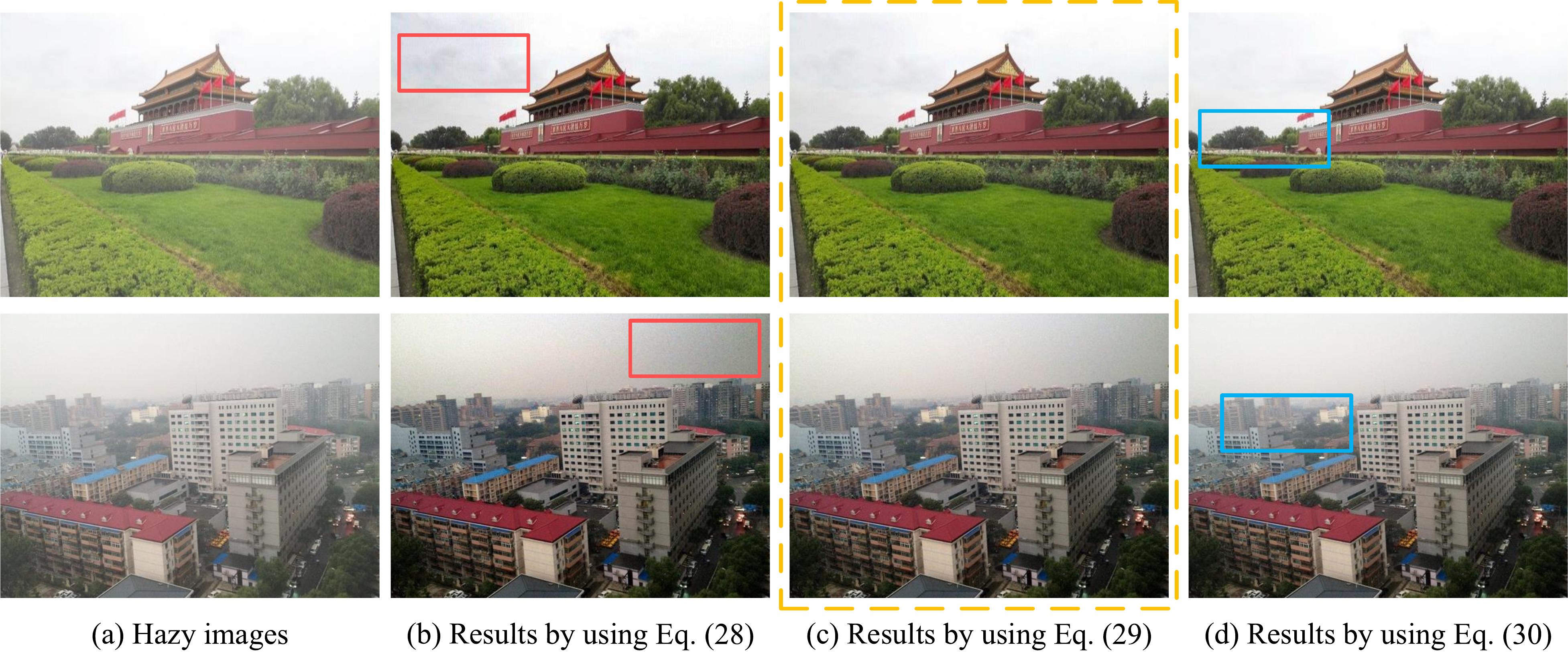}
		\caption{Comparison of dehazing results by using different weight functions.}\label{fig_function_choice_2}
	\end{center}
\end{figure}

\begin{figure*}
	\begin{center}
		\includegraphics[width=7in]{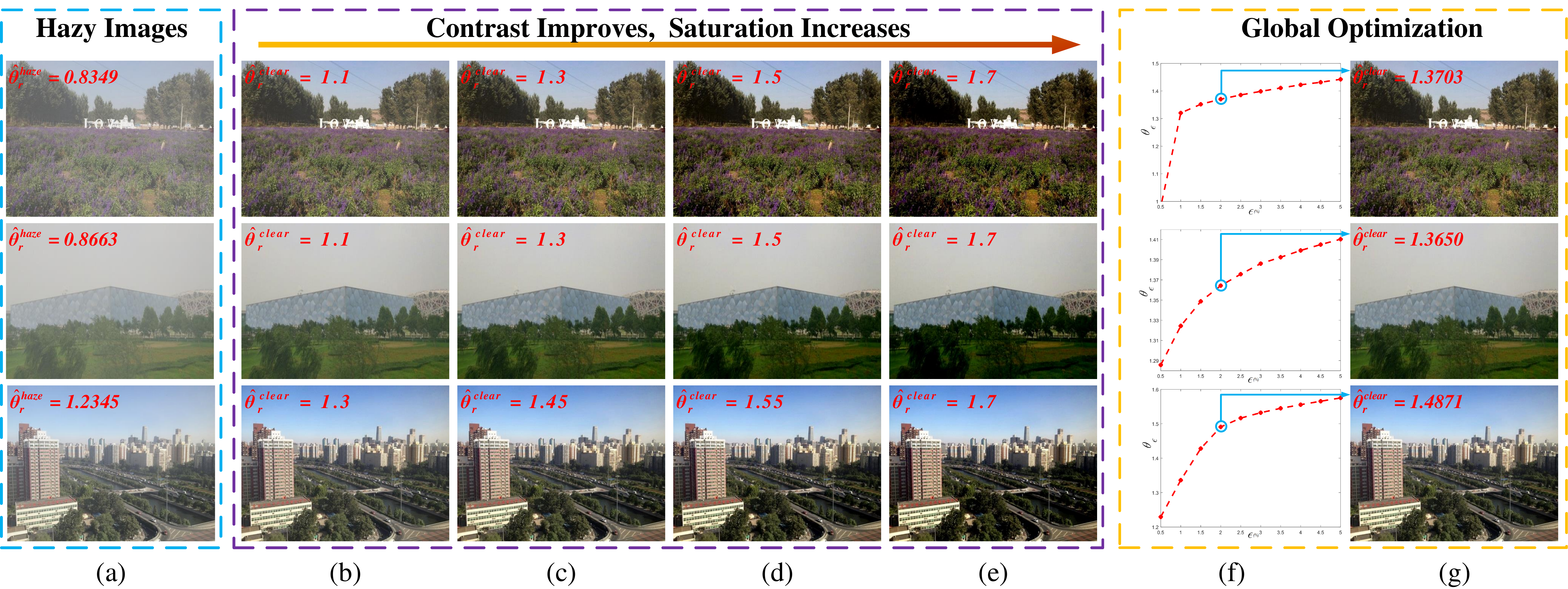}
		\caption{Demonstration of the impact of $\hat \theta _r^{clear}$. \textbf{(a):} Hazy images. \textbf{(b)-(e):} Dehazing results with increasing $\hat \theta _r^{clear}$. \textbf{(f)-(g):} Global optimization process and corresponding results. }\label{fig_global_optimization}
	\end{center}
\end{figure*}
\begin{figure}
	\begin{center}
		\includegraphics[width=3.2in]{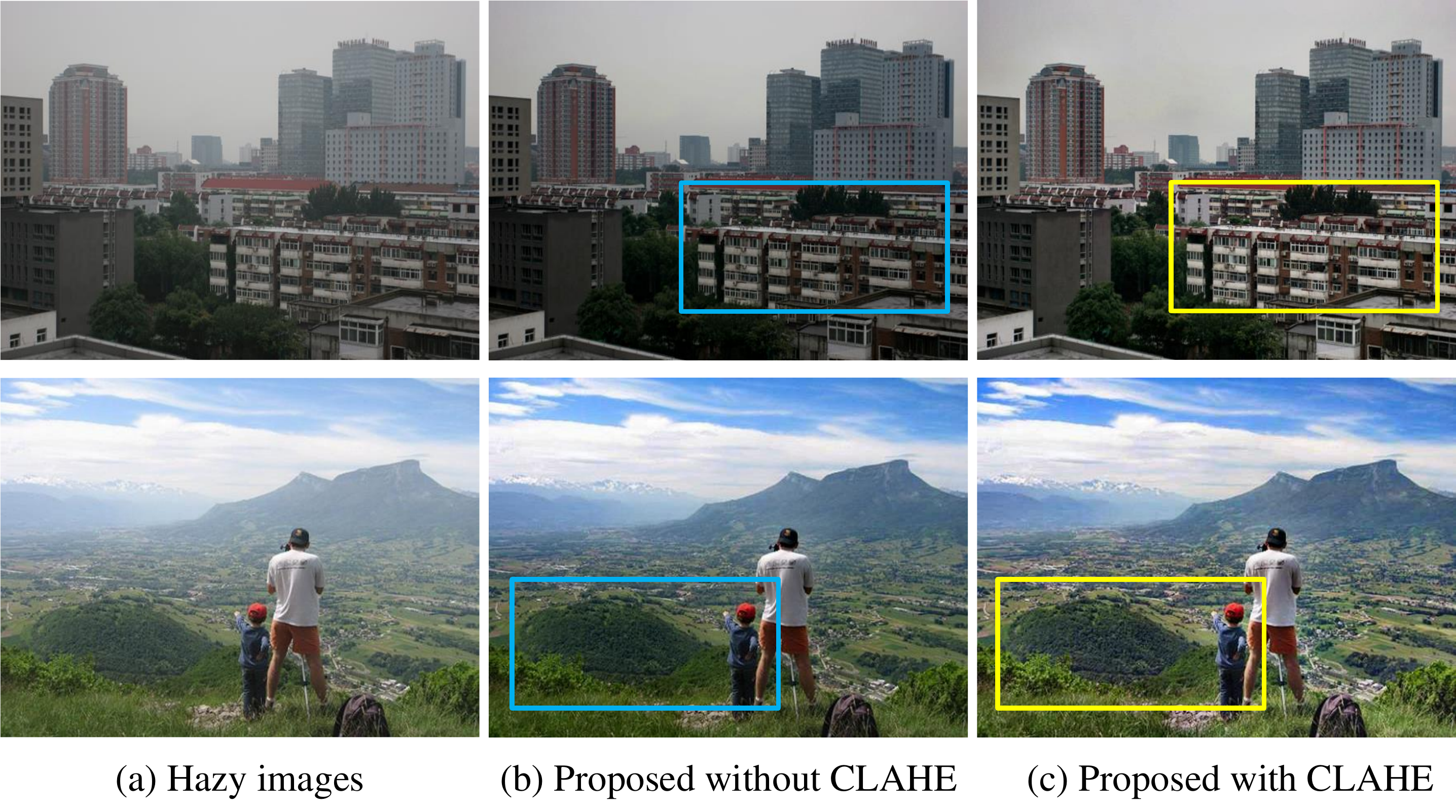}
		\caption{Dehazing results in scenes with dim lighting.}\label{fig_clahe}
	\end{center}
\end{figure}
\begin{figure*}
	\begin{center}
		\includegraphics[width=7in]{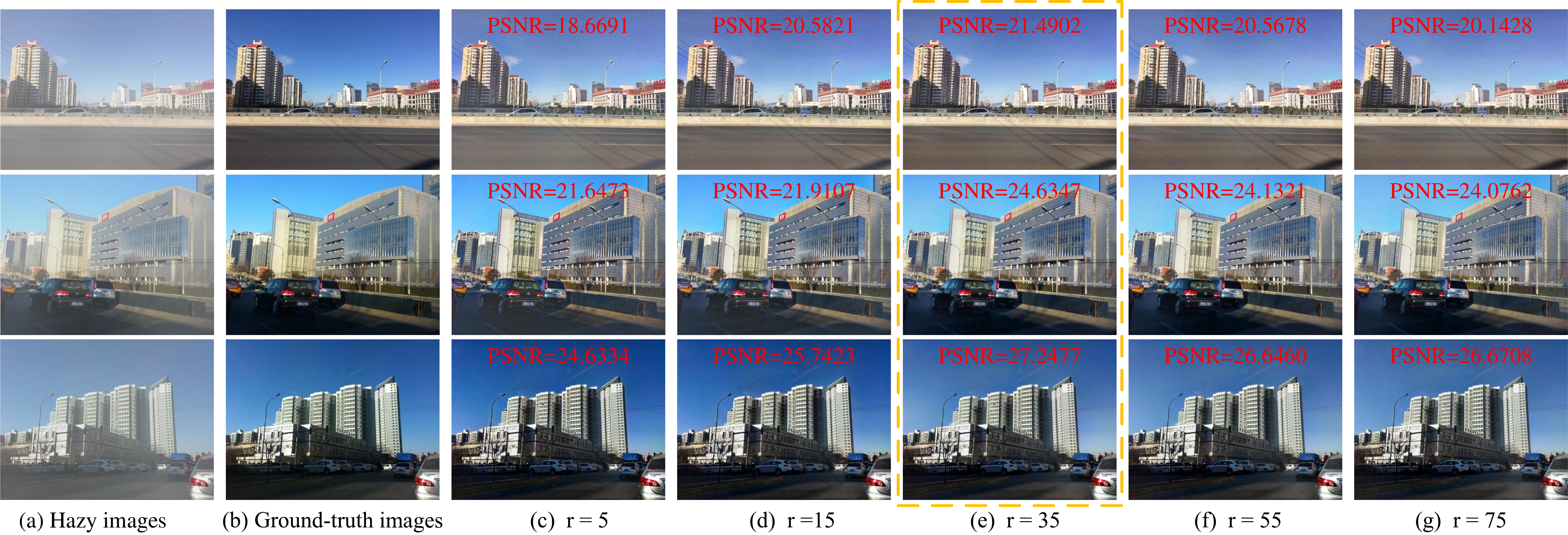}
		\caption{Dehazing results obtained by using different patch sizes.}\label{fig_patchsize}
	\end{center}
\end{figure*}
\begin{figure*}
	\begin{center}
		\includegraphics[width=7in]{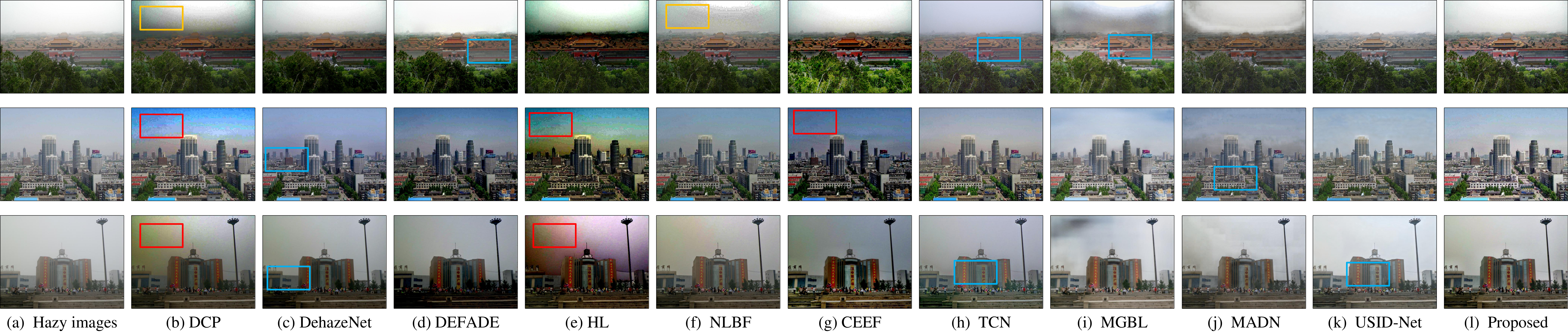}
		\caption{Comparison of the dehazing results in hazy images with large sky regions.}\label{fig_sky_regions}
	\end{center}
\vspace{-3mm}
\end{figure*}

\subsection {Estimation for Atmospheric Light}

Methods for estimating atmospheric light have been proposed in \cite{darkchannel}, \cite{cap}, and \cite{Tang}. The estimated values yielded by these approaches are barely different in most cases because they all estimate atmospheric light from the distant pixels, which usually correspond to the sky regions. In this work, atmospheric light is estimated by using the method in \cite{Tang}.

\subsection {Image Restoration}
Once ${\theta _r^{clear}}(x,y)$ value is obtained, the transmission can be directly obtained from Eq. (\ref{eq.6}), which can be expressed as:
\begin{equation}\label{eq.35}
	{t_r}(x,y) = \frac{{\theta _r^{haze}(x,y)}}{{\theta _r^{clear}(x,y)}}.
\end{equation}
Previously, Eq. (\ref{eq.12}) assumes that depth is constant in a local patch, which may lead to blocking artifacts. Thus, guided filter is employed to smooth the estimated transmission as:
\begin{equation}\label{eq.36}
	{t_g}(x,y) = G({t_r}(x,y)),
\end{equation}
where $G( \cdot )$ is the guided filter operator. Substituting Eq. (\ref{eq.36}) and atmospheric light ${\emph{A}}$ into Eq. (\ref{eq.5}), the dehazing results can be obtained. Note that the dehazing process proposed above is based on the atmospheric scattering model, which assumes that atmospheric light is globally constant. However, this assumption may fail in real world scenes with dim lighting and thus lead to local dark illumination \cite{local_dark_1,local_dark_2}. Therefore, contrast limited adaptive histogram equalization (CLAHE) \cite{CLAHE} is implemented in the obtained dehazing results to compensate for this drawback. The dehazing results with/without CLAHE are shown in Figs. \ref{fig_clahe}(b) and \ref{fig_clahe}(c), respectively. For clarity, the detailed procedure of the proposed method is given in Algorithm 1.
\begin{table}
	\begin{center}
		\setlength{\tabcolsep}{0.005mm}{
			\renewcommand{\arraystretch}{1.2}
			\normalsize
			\begin{tabular}{|c|c|c|c|c|c|}
				\bottomrule[1.5pt]
				\multicolumn{6}{l}{\textbf{Algorithm 1:} Overall Procedure of the Proposed Method}        \\
				\bottomrule
				\multicolumn{6}{l}{\textbf{Input:} Hazy image $\boldsymbol{I}$.}  \\
				\multicolumn{6}{l}{\textbf{Parameter setting:} $r=35$, $\varepsilon=0.02$.}\\
				\multicolumn{6}{l}{\textbf{Begin}}	\\
				\multicolumn{6}{l}  {\quad  Step 1: Estimate ${\emph{A}}$ from H via \cite{Tang}.}      \\
				\multicolumn{6}{l}  {\quad Step 2: Calculate $\theta _r^{haze}(x,y)$ using Eqs. (\ref{eq.3}) and (\ref{eq.12}).}\\      
				\multicolumn{6}{l}{\quad  Step 3: Compute  $z(x,y)$ using Eq. (\ref{eq.27}).}      \\
				\multicolumn{6}{l}{\quad Step 4: Calculate $w(x,y)$ using Eqs. (\ref{eq.20}) and (\ref{eq.29}).} \\
				\multicolumn{6}{l}{\quad Step 5: Determine  $\hat \theta _r^{clear}$ using Eq. (\ref{eq.34}).}    \\
				\multicolumn{6}{l}{\quad Step 6: Estimate ${\theta _r^{clear}}(x,y)$ using Eq. (\ref{eq.19}).}     \\
				\multicolumn{6}{l}{\quad Step 7: Calculate  ${t_g}(x,y)$ using Eqs. (\ref{eq.35}) and (\ref{eq.36}).}          \\
				\multicolumn{6}{l}{\quad Step 8: Restore the haze-free image ${\emph{J}}$ using Eq. (\ref{eq.5}).}       \\
				\multicolumn{6}{l}{\quad Step 9: Apply CLAHE \cite{CLAHE}.}       \\
				\multicolumn{6}{l}{\textbf{End}} \\
				\multicolumn{6}{l}{\textbf{Output:} Restored image $\boldsymbol{J}$.}  \\                        
				\bottomrule
		\end{tabular}}
	\end{center}
\end{table}

\section{EXPERIMENTS }

In this section, we provide detailed analyses and comparison
experiments to facilitate the comprehensive evaluation of the
proposed method. For a comprehensive comparison, the compared methods include five prior-based methods (DCP\cite{darkchannel}, DEFADE\cite{defade}, HL \cite{hl}), NLBF\cite{nlbf} and CEEF \cite{ceef}) and five learning-based methods (DehazeNet \cite{dehazenet}, TCN \cite{tcn}, MGBL \cite{mgbl}, MADN\cite{madn}, and USID-Net\cite{usid}). The testing scenes consist of real benchmark hazy images and common synthetic hazy datasets (O-HAZE \cite{ohaze}, I-HAZE \cite{ihaze}, HSTS\cite{RESIDE}, and SOTS \cite{RESIDE} datasets). For the evaluation of real hazy images without corresponding ground-truth, we employ the non-reference metrics, i.e., the quality of the contrast restoration $\bar r$ \cite{metric} and Fog Aware Density Evaluator (FADE\cite{defade}). For the evaluation in benchmark datasets with corresponding datasets, we apply the common reference metrics, i.e., PSNR, SSIM, and CIEDE2000. Additionally, based on the recent dataset BEDDE \cite{bedde}, we employ the Visibility Index (VI) for visibility restoration evaluation and Realness Index (RI) for realness restoration evaluation.
The experiments are implemented on an Intel Core (TM) i5-6500 CPU@ 3.2GHz, and the parameters of each compared dehazing method are set according to corresponding references.

\begin{figure*}
	\begin{center}
		\includegraphics[width=7in]{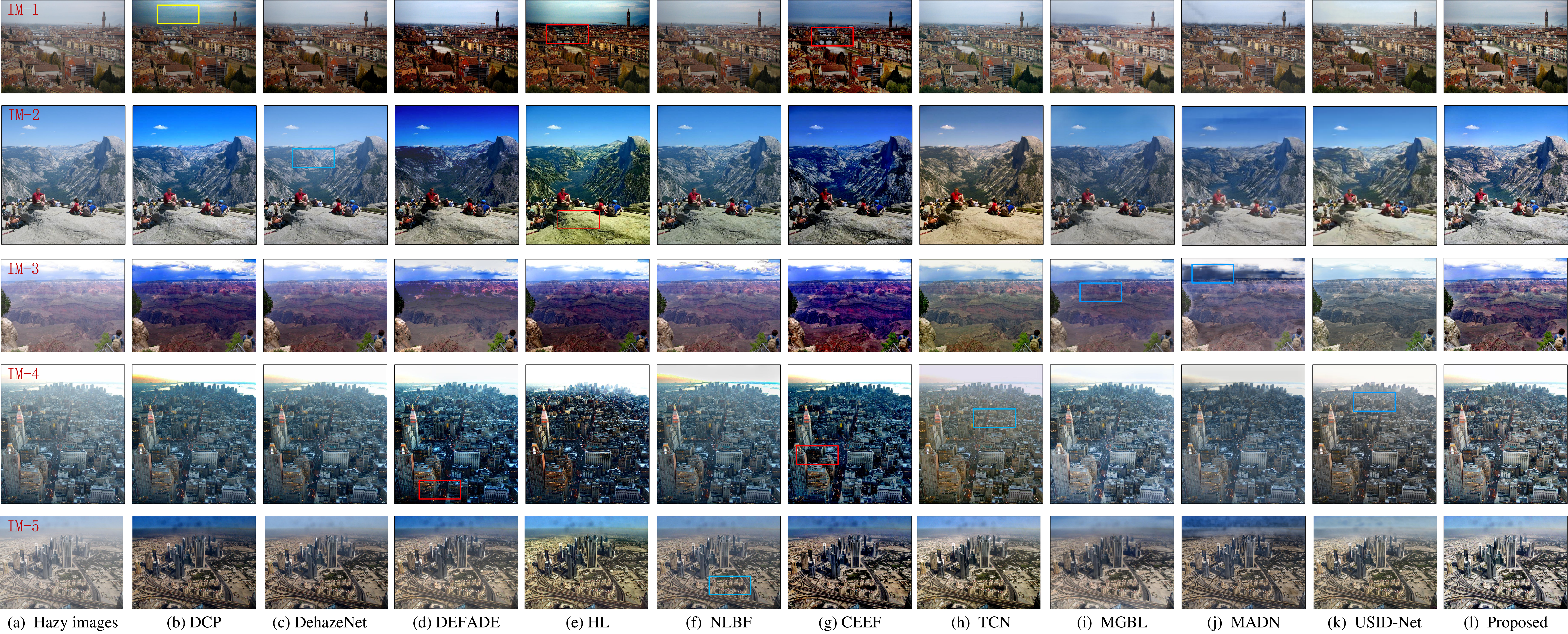}
		\caption{Comparison of dehazing results in real-world benchmark hazy images.}\label{fig_real_world}
	\end{center}
\end{figure*}

\begin{table*}[]
	\begin{center}
		\caption{\textsc{$\bar r$ and FADE Values in the Dehazing Results of Fig. \ref{fig_real_world}}}
		\label{tab_real_world}
		\setlength{\tabcolsep}{0.3mm}{
			\renewcommand{\arraystretch}{1.5}
			\begin{tabular}{c|ccccccccccccc}
				\hline
				{\color[HTML]{333333} Metric} & Image                       & Hazy   & DCP\cite{darkchannel} & DehazeNet\cite{dehazenet} & DEFADE\cite{defade} & HL\cite{hl}        & NLBF\cite{nlbf} & CEEF\cite{ceef}                    & TCN\cite{tcn}                    & MGBL\cite{mgbl} & MADN\cite{madn} & USID-Net\cite{usid} & Proposed       \\ \hline
				& {\color[HTML]{333333} IM-1} & 1      & 1.2125    & 1.1540          & 1.3597       & 1.6288          & 1.1855     & 1.6663                        & {\color[HTML]{333333} 1.4820} & 1.4227     & 1.2874     & 1.4114         & \textbf{1.7563} \\
				& IM-2                        & 1      & 1.3596    & 1.1242          & 1.4440       & 1.6653          & 1.3405     & {\color[HTML]{333333} 1.5608} & {\color[HTML]{333333} 1.2720} & 0.9671     & 1.0513     & 1.3620         & \textbf{1.7589} \\
				& IM-3                        & 1      & 1.4733    & 1.1900          & 1.4808       & 1.7559          & 1.5989     & 2.2251                        & 1.6718                        & 1.1155     & 1.3210     & 1.6911         & \textbf{2.2431} \\
				& IM-4                        & 1      & 1.1610    & 1.1414          & 1.4920       & 1.6172          & 1.2932     & 1.5352                        & 1.3129                        & 1.6434     & 1.0482     & 1.1514         & \textbf{1.7618} \\
				\multirow{-5}{*}{$\bar r$}           & IM-5                        & 1      & 1.6283    & 1.4784          & 2.0084       & 2.1581          & 1.4591     & 2.2375                        & 1.9625                        & 1.4112     & 1.6303     & 1.8283         & \textbf{2.3759} \\ \hline
				& IM-1                        & 0.8475 & 0.2867    & 0.3628          & 0.1999       & \textbf{0.1546} & 0.4412     & \textbf{0.1546}               & 0.4801                        & 0.8376     & 0.5489     & 0.6111         & 0.3420          \\
				& IM-2                        & 0.7524 & 0.4207    & 0.5948          & 0.3005       & 0.2883          & 0.4174     & \textbf{0.2084}               & 0.4306                        & 0.5938     & 0.4399     & 0.4563         & 0.4266          \\
				& IM-3                        & 1.4376 & 0.3958    & 0.7080          & 0.6485       & 0.3140          & 0.5734     & \textbf{0.2044}               & 0.5525                        & 0.7448     & 0.7059     & 0.8487         & 0.3859          \\
				& IM-4                        & 1.4272 & 0.2320    & 0.2528          & 0.1712       & 0.1140          & 0.3205     & \textbf{0.1106}               & 0.3746                        & 0.4322     & 0.4485     & 0.4219         & 0.2393          \\
				\multirow{-5}{*}{FADE}        & IM-5                        & 1.0355 & 0.2239    & 0.3545          & 0.1577       & 0.2296          & 0.4289     & \textbf{0.1387}               & 0.2737                        & 0.6397     & 0.2848     & 0.4905         & 0.2890          \\ \hline
		\end{tabular}}
	\end{center}
\end{table*}
\subsection {Determination of Patch Size}
For the patch size $r$ in Eq. (\ref{eq.12}), on the one hand, a larger patch size contributes to a more effective extraction for depth order, as shown in Fig. \ref{fig_demonstration}(d). On the other hand, the assumption that local pixels share the same depth may fail when the patch size is too large, leading to performance decline. Fig. \ref{fig_patchsize} gives three examples from the Realistic Single Image Dehazing (RESIDE) dataset \cite{RESIDE} by using different $r$ values, and the peak signal to noise ratio (PSNR) is adopted for performance evaluation. It is observed that $r=35$ provides the higher scores, which implies its better ability to recover real structure and details. Therefore, in the following experiments, $r$ is set as 35. It should be pointed out that $r=35$ may not be optimal for all images, but appears to achieve satisfactory results in most scenes.

\subsection {Comparison in Hazy Images With Large Sky Regions}

Image dehazing is more challenging in hazy images with large sky regions, in which amplified noise and distortion colors are prone to severely damage the restoration quality. Fig. 11 gives the visual comparison of different methods in hazy images with large sky regions.  

As can be observed in Figs. \ref{fig_sky_regions}(b)-(k), some methods obtain a satisfactory haze removal by introducing distorted colors (see the red rectangles in Fig. \ref{fig_sky_regions}) or amplified noise (see the yellow rectangles in Fig. \ref{fig_sky_regions}). Correspondingly, some methods obtain a high color fidelity with the cost of producing many haze residues (see the blue rectangles in Fig. \ref{fig_sky_regions}), resulting in a limited restoration quality. In comparison, the proposed method can achieve impressive haze removal while performing remarkable noise suppression and color preservation, thereby bringing significant visibility improvement, as shown in Fig. \ref{fig_sky_regions}(l). Moreover, benefiting from the continuity of the proposed transformation function model, the restored transition between sky and non-sky regions in the dehazing results is also considerably smooth.

\subsection {Qualitative Comparison in Real-World Benchmark Images}
For qualitative comparison, five different types of benchmark images collected from the real world are adopted, which are presented in Fig. \ref{fig_real_world}(a), and the dehazing results of different methods are shown in Figs. \ref{fig_real_world}(b)-(l).

As shown in Fig. \ref{fig_real_world}(b), DCP \cite{darkchannel} is efficient in removing haze in most scenes. However, the noise amplification in sky regions cannot be neglected, as shown in the yellow rectangle in Fig. \ref{fig_real_world}(b). DehazeNet \cite{dehazenet} performs well in mist haze, but suffers from haze residues when facing dense haze (see the blue rectangles in Fig. \ref{fig_real_world}(c)). DEFADE \cite{defade} improves the ability to restore details. However, the recovered color is too dim in low-light regions, as shown in the red rectangles in Fig. \ref{fig_real_world}(d). HL \cite{hl} can remove haze thoroughly. However, over-enhancement is also introduced (see the red rectangles in Fig. \ref{fig_real_world}(e)). Due to the fine control of the transmission boundary, NLBF \cite{nlbf} successfully suppresses the artifacts but results in a limited dehazing performance, as shown in the blue rectangle in Fig. \ref{fig_real_world}(f). Benefiting from the improvement of fusion strategy, CEEF \cite{ceef} achieves better visibility but exhibits over-saturation in some regions (see the red rectangles in Fig. \ref{fig_real_world}(g)). With the refinement for dark and shadow regions, TCN \cite{tcn} obtains dehazing results with increased brightness. Regrettably, the dehazing performance is not satisfactory when the estimated regions are not inaccurate, which can be seen in the blue rectangles in Fig. \ref{fig_real_world}(h). MGBL \cite{mgbl} effectively avoids the introduction of color distortion, but fails to obtain improved visibility in dense haze, as shown in the blue rectangles in Fig. \ref{fig_real_world}(i). A similar situation can also be observed in the dehazing results of MADN \cite{madn} and USID-Net \cite{usid}, in which color fidelity is obtained at the cost of producing many haze residues, resulting in impaired visibility, which can be seen in the blue rectangles in  Figs. \ref{fig_real_world}(j) and \ref{fig_real_world}(k), respectively. Comparatively, the proposed method better reveals potential structure and vivid color without the production of undesired artifacts and color distortion.

\begin{table*}[]
	\begin{center}
		\caption{\textsc{Quantitative Comparison Between the Proposed Method and Other State-of-the-Art Methods on O-HAZE, I-HAZE, HSTS, and SOTS (outdoor set) Datasets Using PSNR, SSIM, and CIEDE2000}}
		\label{tab_synthetic}
		\setlength{\tabcolsep}{0.3mm}{
			\renewcommand{\arraystretch}{1.5}
			\begin{tabular}{c|cccccccccccc}
				\hline
				Dataset                  & Metric    & DCP\cite{darkchannel} & DehazeNet\cite{dehazenet} & DEFADE\cite{defade} & HL\cite{hl}        & NLBF\cite{nlbf} & CEEF\cite{ceef}                    & TCN\cite{tcn}                    & MGBL\cite{mgbl} & MADN\cite{madn} & USID-Net\cite{usid} & Proposed                                \\ \hline
				& PSNR      & 16.3433 & 15.5106                                 & 15.6069 & 15.6912 & 16.7201                                 & 13.9424 & 14.7832                                & 14.9930                                 & {\color[HTML]{34CDF9} \textbf{17.2146}} & 15.5918                                & {\color[HTML]{FE0000} \textbf{17.3147}} \\
				& SSIM      & 0.4794  & 0.4497                                  & 0.3817  & 0.5430  & 0.4886                                  & 0.2883  & {\color[HTML]{34CDF9} \textbf{0.5477}} & 0.3969                                  & 0.4096                                  & 0.5240                                  & {\color[HTML]{FE0000} \textbf{0.5847}}  \\
				\multirow{-3}{*}{O-HAZE} & CIEDE2000 & 17.0019 & 16.6517                                 & 19.7347 & 17.1270 & {\color[HTML]{34CDF9} \textbf{15.0263}} & 22.3689 & 17.1086                                & 17.7804                                 & 16.6940                                 & 16.0474                                 & {\color[HTML]{FE0000} \textbf{14.1502}} \\ \hline
				& PSNR      & 13.1961 & 15.8003                                 & 16.0261 & 15.6530 & 16.5040                                 & 15.5721 & 16.6034                                & {\color[HTML]{34CDF9} \textbf{16.8592}} & 15.4370                                 & 16.5037                                 & {\color[HTML]{FE0000} \textbf{17.9898}} \\
				& SSIM      & 0.5061  & 0.5817                                  & 0.5978  & 0.6088  & {\color[HTML]{34CDF9} \textbf{0.6378}}  & 0.5566  & 0.5728                                 & 0.6262                                  & 0.6067                                  & 0.5670                                  & {\color[HTML]{FE0000} \textbf{0.6934}}  \\
				\multirow{-3}{*}{I-HAZE} & CIEDE2000 & 18.5827 & 14.0629                                 & 14.1285 & 14.5911 & 12.7354                                 & 15.4363 & 14.2158                                & {\color[HTML]{34CDF9} \textbf{12.6508}} & 14.6021                                 & 13.3430                                 & {\color[HTML]{FE0000} \textbf{11.0443}} \\ \hline
				& PSNR      & 17.1887 & {\color[HTML]{34CDF9} \textbf{24.5041}} & 17.3229 & 17.6170 & 18.8010                                 & 16.2349 & 19.8511                                & 18.9224                                 & 18.9471                                 & {\color[HTML]{FE0000} \textbf{25.5120}} & 22.3724                                 \\
				& SSIM      & 0.8089  & {\color[HTML]{FE0000} \textbf{0.9176}}  & 0.7843  & 0.7917  & 0.8428                                  & 0.7124  & 0.7692                                 & 0.8419                                  & 0.8661                                  & 0.8462                                 & {\color[HTML]{34CDF9} \textbf{0.8855}}  \\
				\multirow{-3}{*}{HSTS}   & CIEDE2000 & 10.1881 & {\color[HTML]{FE0000} \textbf{4.9260}}  & 11.6041 & 10.4795 & 8.4251                                  & 11.8518 & 9.4623                                 & 8.8201                                  & 8.5953                                  & {\color[HTML]{34CDF9} \textbf{5.3790}}  & 6.5447                                  \\ \hline
				& PSNR      & 17.5736 & {\color[HTML]{34CDF9} \textbf{22.7229}} & 18.0681 & 18.4898 & 20.0290 & 15.9704 & 19.8904 & 18.2478 & 19.0631                                & {\color[HTML]{FE0000} \textbf{23.8092}} & 22.0602                                 \\
				& SSIM      & 0.8143  & 0.8578                                  & 0.8016  & 0.8298  & {\color[HTML]{34CDF9} \textbf{0.8711}}  & 0.6834  & 0.7322  & 0.8379  & 0.8621 & 0.8174                                  & {\color[HTML]{FE0000} \textbf{0.8922}}  \\
				\multirow{-3}{*}{SOTS}   & CIEDE2000 & 10.8241 & {\color[HTML]{FE0000} \textbf{5.8634}}  & 10.2653 & 9.8122  & 7.3395  & 12.3574 & 9.6257  & 9.6257  & 8.4707                                 & 6.7435                                  & {\color[HTML]{34CDF9} \textbf{6.6182}}                                  \\ \hline		
		\end{tabular}}
	\end{center}
\end{table*}

\begin{table*}[]
	\begin{center}
		\caption{\textsc{Quantitative Comparison Between the Proposed Method and Other State-of-the-Art Methods on BeDDE Dataset Using VI and RI}}
		\label{tab_bedde}
		\setlength{\tabcolsep}{0.83mm}{
			\renewcommand{\arraystretch}{1.5}
			\begin{tabular}{c|ccccccccccc}
				\hline
				Metric & DCP\cite{darkchannel} & DehazeNet\cite{dehazenet} & DEFADE\cite{defade} & HL\cite{hl}        & NLBF\cite{nlbf} & CEEF\cite{ceef}                    & TCN\cite{tcn}                    & MGBL\cite{mgbl} & MADN\cite{madn} & USID-Net\cite{usid} & Proposed                                    \\ \hline
				VI     & {\color[HTML]{FE0000} \textbf{0.9111}} & 0.8902                                 & ——     & 0.8278 & 0.8910 & 0.8736  & 0.8960 & 0.8672 & 0.8495 & 0.8521 & {\color[HTML]{34CDF9} \textbf{0.9012}} \\ \hline
				RI     & 0.9654                                 & {\color[HTML]{34CDF9} \textbf{0.9718}} & ——     & 0.9557 & 0.9702 & 0.96709 & 0.9617 & 0.9630 & 0.9625 & 0.9622 & {\color[HTML]{FE0000} \textbf{0.9725}} \\ \hline
		\end{tabular}}
	\end{center}
\end{table*}

\begin{table*}[]
	\begin{center}
		\caption{\textsc{Run Time on Various Resolution Images (Unit$:$ Second)}}
		\label{tab_runtime}
		\setlength{\tabcolsep}{0.68mm}{
			\renewcommand{\arraystretch}{1.5}
			\begin{tabular}{c|ccccccccccc}
				\hline
				& DCP\cite{darkchannel} & DehazeNet\cite{dehazenet} & DEFADE\cite{defade} & HL\cite{hl}        & NLBF\cite{nlbf} & CEEF\cite{ceef}                    & TCN\cite{tcn}                    & MGBL\cite{mgbl} & MADN\cite{madn} & USID-Net\cite{usid} & Proposed                              \\
				\multirow{-2}{*}{Resolution} & Matlab   & Matlab    & Matlab   & Matlab  & Matlab                                 & Matlab                                 & Tensorflow & Pytorch                                 & Pytorch  & Pytorch  & Matlab                                 \\ \hline
				400*600                      & 14.3903  & 2.0447    & 12.8527  & 4.8452  & 0.9234                                 & {\color[HTML]{34CDF9} \textbf{0.8639}} & 1.2504     & 1.7193                                  & 4.2733   & 3.6345   & {\color[HTML]{FE0000} \textbf{0.1939}} \\
				600*800                      & 27.5310  & 4.2538    & 24.1943  & 6.1620  & {\color[HTML]{34CDF9} \textbf{1.7325}} & 1.8117                                 & 2.4525     & 2.4534                                  & 9.8045   & 7.3815   & {\color[HTML]{FE0000} \textbf{0.3647}} \\
				800*1000                     & 21.0841  & 6.6924    & 40.2225  & 7.8164  & {\color[HTML]{34CDF9} \textbf{2.6985}} & 2.8627                                 & 4.2214     & 3.1890                                  & 18.0214  & 12.1151  & {\color[HTML]{FE0000} \textbf{0.5571}} \\
				1000*1200                    & 80.1020  & 10.3083   & 60.7340  & 10.9404 & {\color[HTML]{34CDF9} \textbf{3.9942}} & 4.2601                                 & 6.2811     & 4.1947                                  & 23.2699  & 17.6227  & {\color[HTML]{FE0000} \textbf{0.8722}} \\
				1200*1600                    & 165.1350 & 16.1073   & 145.8779 & 13.0515 & {\color[HTML]{333333} 6.7551}          & 7.0218                                 & 10.0010    & {\color[HTML]{34CDF9} \textbf{6.4712}}  & 38.6203  & 29.0681  & {\color[HTML]{FE0000} \textbf{1.4549}} \\
				1600*2000                    & ——    & 28.1794   & ——     & 22.5676 & 11.0415                                & 11.3232                                & 16.6101    & {\color[HTML]{34CDF9} \textbf{7.2503}}  & 64.7941  & 50.0431  & {\color[HTML]{FE0000} \textbf{2.2896}} \\
				2000*3000                    & ——     & 65.8699   & ——     & 45.3395 & 20.4891                                & 21.8789                                & 30.9844    & {\color[HTML]{34CDF9} \textbf{10.4258}} & 140.9155 & 93.9108  & {\color[HTML]{FE0000} \textbf{4.2049}} \\ \hline
		\end{tabular}}
	\end{center}
\end{table*}

\begin{figure*}[]
	\begin{center}
		\includegraphics[width=6.6in]{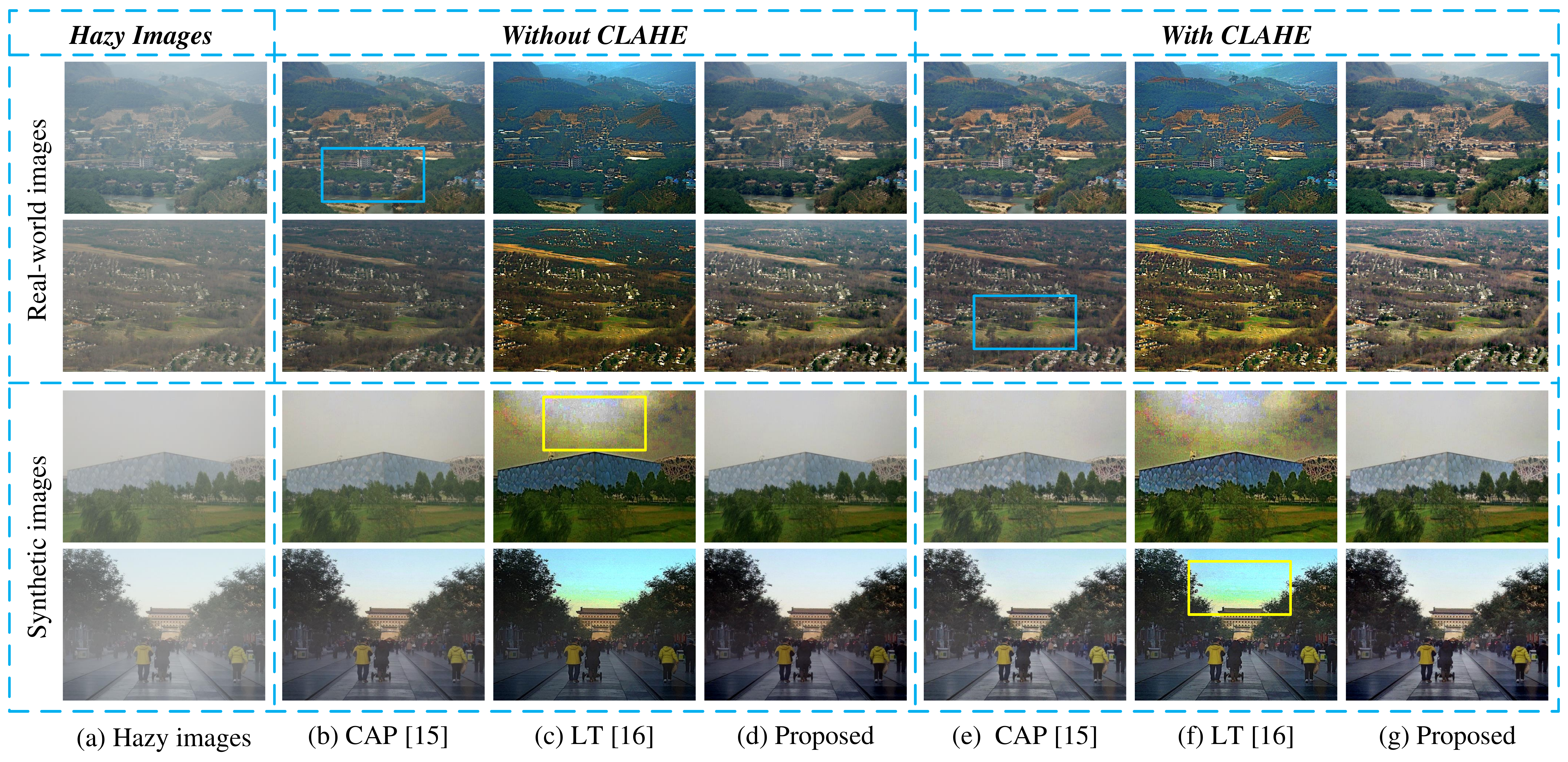}
		\caption{Comparison of dehazing results with CAP \cite{cap} and LT\cite{lt} in both real-world and synthetic hazy images.}\label{fig_other_depthbased}
	\end{center}
\vspace{-3mm}
\end{figure*}
\begin{figure}
	\begin{center}
		\includegraphics[width=3.2in]{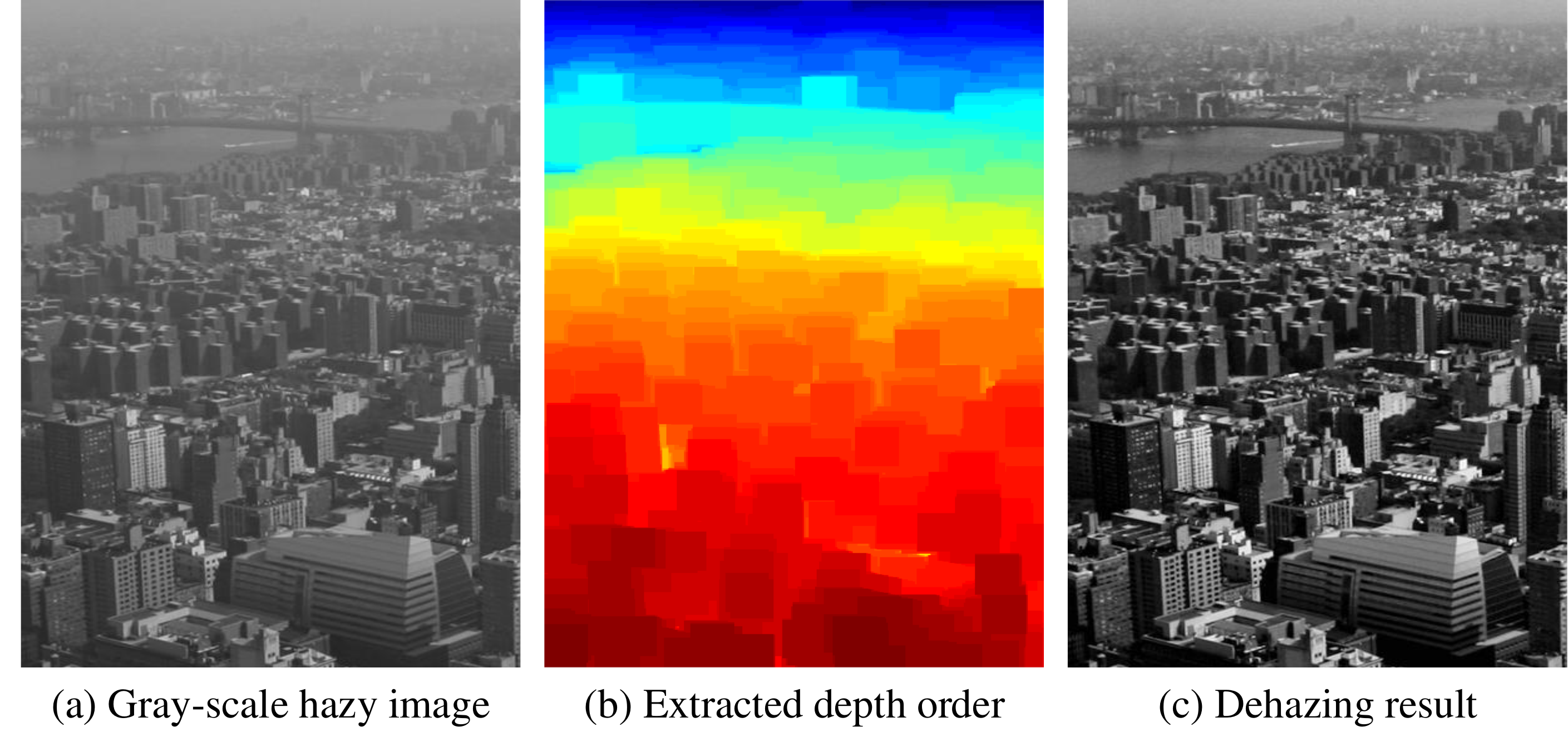}
		\caption{Dehazing result in gray-scale hazy image.}\label{fig_gray_image}
	\end{center}
\vspace{-6mm}
\end{figure}

\subsection {Quantitative Evaluation}

For a more convincing evaluation, five widely used metrics are adopted for the quantitative evaluation. Specifically, the quality of the contrast restoration $\bar r$ \cite{metric} and Fog Aware Density Evaluator (FADE) \cite{defade} are used for the real-world hazy images without corresponding clear images. Additionally, PSNR, structural similarity index measure (SSIM) \cite{ssim}, and CIEDE2000 color difference \cite{ciede} are employed for synthetic images that have corresponding ground-truth images. Note that a larger $\bar r$ value or a lower FADE value indicates a better dehazing performance.

The scores of different methods in terms of $\bar r$ and FADE in the dehazing results of Fig. \ref{fig_real_world} are given in Table \ref{tab_real_world}, in which the best scores are highlighted. As presented in Table \ref{tab_real_world}, the proposed method yields the best scores of $\bar r$ in all hazy images, which indicates its superiority in increasing visibility and revealing structure. Note that, although HL \cite{hl} and CEEF \cite{ceef} produce the best values of FADE, the dehazing results obtained by these methods appear too dark in some regions and are followed by over-saturation (see the rectangles in Figs. \ref{fig_real_world}(e) and \ref{fig_real_world}(g)), which can reduce the level of fog aware density and thus lead to a lower value of FADE. However, these dehazing results are visually unpleasing. By comparison, the proposed method achieves much smaller FADE values than hazy images without the production of color distortion, and thus obtains the impressive results.

To further evaluate the dehazing performance, we compare the dehazing performance on three benchmark hazy datasets, i.e., O-HAZE \cite{ohaze}, I-HAZE \cite{ihaze}, HSTS \cite{RESIDE}, and SOTS (outdoor set) datasets \cite{RESIDE}, which provide paired hazy images and haze-free images. The scores of different methods on these three datasets regarding PSNR, SSIM, and CIEDE2000  are presented in Table \ref{tab_synthetic}, in which the best and second-best scores are highlighted in red and blue, respectively. Note that the images of O-HAZE and I-HAZE are resized to the suggested size in \cite{ohaze} to facilitate the implementation of methods requiring high memory requirements. 

Table \ref{tab_synthetic} shows that the proposed method attains the best values of PSNR, SSIM, and CIEDE2000 on both O-HAZE and I-HAZE datasets, which implies its superiority in recovering real structure and actual color. This improvement can be attributed to that the proposed method embeds the extracted depth order into the devised dehazing model and provides effective guidance for the dehazing process, thereby bringing a higher restoration quality. Benefiting from the strong learning ability of CNN, DehazeNet \cite{dehazenet} and USID-Net \cite{usid} appear to obtain most best or second-best scores on HSTS dataset, and the proposed method achieves the second-best score regarding SSIM metric. However, it should be noted that the dehazing performance of DehazeNet and USID-Net is hard to maintain when facing real-world hazy images due to over-fitting, which can be observed in Figs. \ref{fig_real_world}(c) and \ref{fig_real_world}(k). In comparison, the proposed method performs better haze removal in real-world images and attains competitive results on synthetic datasets, which indicates its comprehensive improvement in dehazing performance. For the outdoor set of SOTS dataset, the proposed methods also achieves competitive performance, which further validates its generalization.

Additionally, to facilitate diverse comparison,  we also compare the dehazing performance of different methods in a recent benchmark dataset BeDDE\cite{bedde}, which contains two specially designed metrics, i.e., Visibility Index (VI) for visibility restoration evaluation and Realness Index (RI) for realness restoration evaluation. Note that VI and RI are the customized evaluation metrics for BeDDE dataset. A larger VI value means better visibility restoration and a larger RI value implies better realness recovery. The qualitative results of different methods on BeDDE are presented in Table \ref{tab_bedde}. 

As can be observed in Table \ref{tab_bedde}, the proposed method achieves the best score in RI and the second-best score in VI, which further implies its superiority in improving visibility and preserving original structure and real color. Note that DCP \cite{darkchannel} has a natural advantage in VI metric because the transmission used to calculate VI value is estimated by DCP. Additionally, the result of the fusion-based method DEFADE\cite{defade} is omitted, because it fails to run in the hazy images of BeDDE, which contain totally dark local areas.

\subsection {Run Time}
Real-time applications or high-resolution images are picky about the computational efficiency. Therefore, the comparison of different methods regarding run time is carried out in images with various resolutions. For an unbiased evaluation, all methods are implemented on the same CPU, and the run time at each resolution is taken from the average value of five simulations from the resized images of Fig. 12. The results of each method are summarized in Table \ref{tab_runtime}. Note that the symbol “ — ” denotes “ out of memory ”. It is observed that the proposed method achieves the best computational efficiency at all resolutions, which demonstrates its better competitiveness for real-time tasks.

\subsection {Link with Previous Work}
The proposed method is developed on the extracted depth order from hazy images. Depth-relevant color characteristics have been demonstrated to be useful for image dehazing in previous work, i.e., CAP\cite{cap} and LT \cite{lt}. Therefore, we specifically conduct a visual comparison to intuitively demonstrate the dehazing performance difference, which is shown in Fig. \ref{fig_other_depthbased}.

As shown in Fig. \ref{fig_other_depthbased}, CAP \cite{cap} fails to remove the haze in distant regions and dense haze regions and thus produce many haze residues in the dehazing results (see the blue rectangles in Figs. \ref{fig_other_depthbased}(b) and \ref{fig_other_depthbased}(e)). This poor performance in severely degraded regions is attributed to its limited estimation performance for scene depth. For LT \cite{lt}, the transmission values are under-estimated in the distant areas, leading to over-exposure and noise amplification (see the yellow rectangles in Figs. \ref{fig_other_depthbased}(c) and \ref{fig_other_depthbased}(f)). In contrast, the proposed method effectively alleviates the problems of haze residue and noise amplification in both real-world and synthetic hazy images, regardless of whether CLAHE is applied. It mainly benefits from the avoidance of direct estimation for depth value and the effective exploitation for depth order in the transmission estimation process, thereby bringing a comprehensive improvement in dehazing performance. 

Moreover, it is worth mentioning that the proposed method also works well for gray-scale hazy images; it is easy to verify that the proposed depth order extraction strategy also holds for the luminance channel in gray images. A dehazing example for gray-scale hazy image is presented in Fig. \ref{fig_gray_image}.

\section{Conclusion}
In this paper, we propose a depth order guided single image dehazing method. The proposed method utilizes the consistency of depth perception as the global constraint and provides powerful guidance for image dehazing, thereby achieving dehazing results with higher quality. Specifically, a simple yet effective approximation strategy is proposed to extract the depth order in hazy images, which gives the reference of depth perception. Then, a depth order embedded transformation model is established for transmission estimation, which removes haze while ensuring an unchanged depth order. The embedded depth order, which offers global constraint, brings a more effective utilization of global information, thereby achieving a considerable quality improvement of dehazing results. Additionally, the proposed method does not involve any pre-processing or iteration steps, which enables a highly competitive efficiency. Extensive experiments show that the proposed method achieves a better dehazing performance against most state-of-the-art dehazing methods.

\ifCLASSOPTIONcaptionsoff
  \newpage
\fi

\bibliographystyle{IEEEtran}
\bibliography{Bibliography}

\vfill

\end{document}